\documentclass{article}

\usepackage{microtype}
\usepackage{graphicx}
\usepackage{booktabs} 

\usepackage{hyperref}

\usepackage[accepted]{icml2024}

\usepackage{amsmath}
\usepackage{amssymb}
\usepackage{mathtools}
\usepackage{amsthm}
\usepackage{soul}

\usepackage[capitalize,noabbrev]{cleveref}

\theoremstyle{plain}
\newtheorem{theorem}{Theorem}[section]

\newtheorem{lemma}[theorem]{Lemma}

\theoremstyle{definition}

\theoremstyle{remark}
\newtheorem{remark}[theorem]{Remark}

\usepackage[textsize=tiny]{todonotes}

\icmltitlerunning{
Environment Design for Inverse Reinforcement Learning
}

\usepackage{amsfonts}
\usepackage{amsmath}
\usepackage{amsthm}
\usepackage{graphicx}
\usepackage{mathrsfs}   
\usepackage{mathtools}  
\usepackage{bbm}
\usepackage{algorithm}  
\usepackage[noend]{algorithmic}
\usepackage{subcaption}
\usepackage{tikz}           
\usetikzlibrary{automata, arrows, arrows.meta, positioning}
\usepackage{enumitem}
\usepackage{arydshln}

\usepackage{dblfloatfix}

\usepackage{hyperref}

\newcommand{\Reals}{\mathbb{R}}
\newcommand{\Naturals}{\mathbb{N}}
\renewcommand{\S}{\mathcal{S}}
\newcommand{\A}{\mathcal{A}}

\renewcommand{\r}{R}

\newcommand{\val}{\mathcal{V}}  

\newcommand{\birl}{\texttt{BIRL}}

\newcommand{\Sim}{\Trans} 
\newcommand{\Q}{{Q}}    
\newcommand{\E}{\mathbb{E}}
\newcommand{\truereward}{\boldsymbol{R}}
\newcommand{\trans}{T}
\newcommand{\basetrans}{{\trans_{\texttt{base}}}}    
\newcommand{\Trans}{\mathcal{T}}    
\newcommand{\Transdemo}{\Trans_{\texttt{demo}}}
\newcommand{\Transtest}{\Trans_{\texttt{test}}}

\renewcommand{\P}{\mathbb{P}}
\newcommand{\empP}{{\smash{\hat \P}}}

\newcommand{\f}{\mathbf{f}}     
\newcommand{\alltraj}{TBD}  

\newcommand{\bel}{\P}
\newcommand{\data}{\mathcal{D}} 

\newcommand{\bestR}{\bar{R}}

\DeclareMathOperator{\bayesregret}{BR}

\newcommand{\argmax}{\mathop{\rm arg\,max}}

\newcommand{\supp}{\mathop{\rm supp}}
\newcommand{\rhodemo}{\rho_{\texttt{demo}}}
\newcommand{\rhotest}{\rho_{\texttt{test}}}

\newcommand{\edbirl}{\texttt{ED-BIRL}} 
\newcommand{\drbirl}{\texttt{DR-BIRL}}

\newcommand{\edairl}{\texttt{ED-AIRL}}
\newcommand{\drairl}{\texttt{DR-AIRL}}
\newcommand{\airl}{\texttt{AIRL}}
\newcommand{\airlme}{\texttt{AIRL-ME}}
\newcommand{\bc}{\texttt{BC}}
\newcommand{\rime}{\texttt{RIME}}
\newcommand{\gail}{\texttt{GAIL}}

\theoremstyle{definition}

\usepackage{array, booktabs}
\newcommand{\PreserveBackslash}[1]{\let\temp=\\#1\let\\=\temp}
\newcolumntype{C}[1]{>{\PreserveBackslash\centering}p{#1}}
\newcolumntype{R}[1]{>{\PreserveBackslash\raggedleft}p{#1}}
\newcolumntype{L}[1]{>{\PreserveBackslash\raggedright}p{#1}}


  
\begin{document}

\twocolumn[
\icmltitle{
Environment Design for
Inverse Reinforcement Learning
}

\icmlsetsymbol{equal}{*}

\begin{icmlauthorlist}
\icmlauthor{Thomas Kleine Buening}{turing,equal}
\icmlauthor{Victor Villin}{neuchatel,equal}
\icmlauthor{Christos Dimitrakakis}{neuchatel}
\end{icmlauthorlist}

\icmlaffiliation{turing}{The Alan Turing Institute, London, UK}
\icmlaffiliation{neuchatel}{Université de Neuchâtel, Neuchâtel, Switzerland}

\icmlcorrespondingauthor{Thomas Kleine Buening}{tbuening@turing.ac.uk}

\icmlkeywords{Inverse Reinforcement Learning, Environment Design}

\vskip 0.3in
]
\printAffiliationsAndNotice{\icmlEqualContribution} 

\begin{abstract}
    
    Learning a reward function from demonstrations suffers from low
    sample-efficiency. Even with
    abundant data, current inverse reinforcement learning methods that focus on learning from a single
    environment can fail to handle slight changes in the environment
    dynamics. We tackle these challenges through adaptive environment
    design. In our framework, the learner repeatedly interacts with
    the expert, with the former selecting environments to identify the reward function as quickly as possible from the expert's demonstrations in said environments. This results in
    improvements in both sample-efficiency and robustness, as we show
    experimentally, for both exact and approximate inference.
    
    
    %
    
\end{abstract}

\section{Introduction}



\begin{figure*}[t]
     \centering
         \begin{subfigure}[b]{0.32\textwidth}
         \centering
         \includegraphics[height=0.43\textwidth]{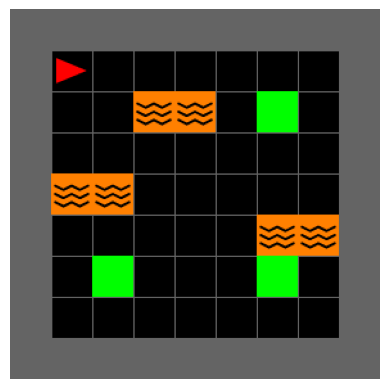} 
         \includegraphics[height=0.43\textwidth]{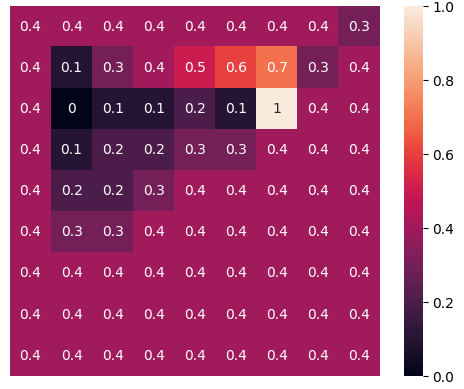}
         \caption{1st round}
     \end{subfigure}
      \hfill
    \begin{subfigure}[b]{0.32\textwidth}
         \centering
         \includegraphics[height=0.43\textwidth]{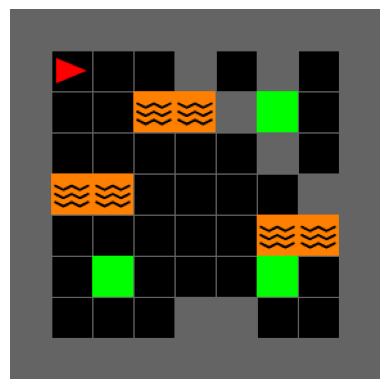}
         \includegraphics[height=0.43\textwidth]{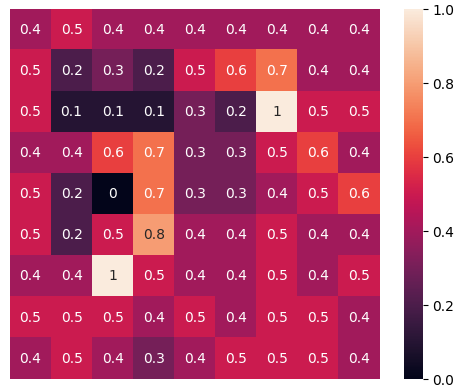}
         \caption{2nd round}
     \end{subfigure}
     \hfill
    \begin{subfigure}[b]{0.32\textwidth}
         \centering
         \includegraphics[height=0.43\textwidth]{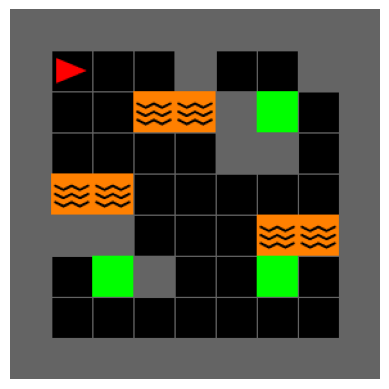}
         \includegraphics[height=0.43\textwidth]{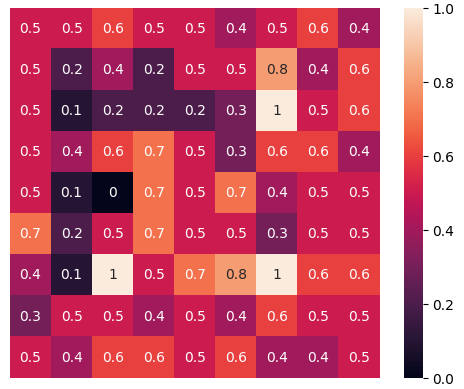}
         \caption{3rd round}
     \end{subfigure}
     \caption{The expert navigates to the closest of the three possible goal squares while avoiding lava in adaptively elected maze environments. For three consecutive rounds (a)-(c), we display the mazes chosen by \texttt{ED-BIRL} (Algorithm~\ref{algorithm:ED-BIRL} in Section~\ref{section:irl}) as well as the current reward estimate after observing an expert trajectory in the current and past mazes. By adaptively designing environments and combining the expert demonstrations, we can recover the locations of all goal and most lava squares. In contrast, from observations in a fixed environment, e.g., repeatedly observing the expert in maze (a), it would be impossible to recover all relevant aspects of the reward function, i.e., the location of the goal squares, as only the nearest goal square would be visited by the expert (repeatedly). Observing the human expert in new and carefully curated environments can lead to a more precise and robust estimate of the unknown reward function. 
     }
     \label{fig:intro}
     \vspace{-0.2cm}
\end{figure*} 


Reinforcement Learning (RL) is a powerful framework for autonomous decision-making in games~\citep{mnih2015human}, continuous control problems~\citep{lillicrap2015continuous}, and robotics~\citep{levine2016end}. 
However, specifying suitable reward functions remains one of the main barriers to the wider application of RL in real-world settings and methods that allow us to communicate tasks without manually defining reward functions could be of great practical value. One such approach is
Inverse Reinforcement Learning (IRL), which aims to find a reward function that explains observed (human) behaviour \citep{russell1998learning, ng2000IRL}.


Much of recent effort in IRL has been devoted to making existing methods more sample-efficient as well as robust to changes in the environment dynamics \citep{arora2021survey}. Sample-efficiency is crucial, as data requires expensive human input. We also need robust estimates of the unknown reward function, so that the resulting policies remain near-optimal, even when the deployed environment dynamics differ from the ones we learned from.


However, recent work has found that IRL methods tend to heavily specialise (``overfit'') to the specific transition dynamics under which the demonstrations were provided, thereby failing to generalise even across minor changes in the environment~\citep{toyer2020magical}. More generally, even with unlimited access to expert demonstrations, we may still fail to learn suitable reward functions from a fixed environment. In particular, prior work has explored the identifiability problem in IRL \citep{cao2021identifiability, kim2021reward}, illustrating the inherent limitations of IRL when learning from expert demonstrations in a single, fixed environment.

In our study, we consider the situation where we can \emph{design} a sequence of \emph{environments}, in which the expert will demonstrate the task. This can either mean slightly modifying a base environment, or selecting an environment from a finite set. Crafting new environments can involve simple adjustments such as relocating objects or adding obstacles, which can be done with little effort and cost. In open-world settings, simply having different task conditions (e.g., different cars, locations, or time-of-day in a vehicle scenario) amounts to a different environment.  


\vspace{-0.2cm}
\paragraph{Contributions.}
We propose algorithms for designing environments in order to infer human goals from demonstrations. This requires two key components: firstly, an \emph{environment design} algorithm  and secondly, an \emph{inference} algorithm for data from multiple environment dynamics. Our hypothesis is that intelligent environment design can significantly improve both sample-efficiency of IRL methods and the robustness of learned rewards against variations in the environment dynamics. 
An example where our approach is applicable is given in Figure~\ref{fig:intro}, where we need to learn the reward function (i.e., the location of the goal and lava squares). 
In summary, our contributions are:
\begin{enumerate}[itemsep=0pt, topsep=0pt, leftmargin=12pt]
    \item An environment design framework that selects informative demo environments for the experts  (Section~\ref{section:problem_formulation}).
    \item An objective based on maximin Bayesian regret to choose environments in a way that compels the expert to provide useful information about the unknown reward function (Section~\ref{section:environment_design}). 
    \item An extension of Bayesian IRL and Maximum Entropy IRL to multiple environments (Section~\ref{section:irl}). We provide concrete implementations for the extensions of MCMC Bayesian IRL~\cite{ramachandran2007BIRL} and Adversarial IRL ($\airl$)~\cite{fu2017learning}. 
    \item We conduct extensive experiments to evaluate our approaches (Section~\ref{section:experiments}). 
    We test learned reward functions in unknown transition dynamics across various environments, including continuous mazes and MuJoCo benchmarks~\cite{todorov2012mujoco}. We compare against several other IRL and imitation learning algorithms, such as Robust Imitation learning with Multiple perturbed Environments ($\rime$)~\cite{chae2022robust}.
    \item Our results illustrate the superior robustness of our algorithms and the effectiveness of the environment design framework. This shows that active environment selection significantly improves both the sample-efficiency of IRL and the robustness of learned rewards (generalisability). 
\end{enumerate}

\section{Related Work}\label{section:related_work}

\paragraph{Inverse Reinforcement Learning.} 

The goal of IRL \citep{russell1998learning, ng2000IRL} is to find a reward function that explains observed behaviour, which is assumed to be approximately optimal. Two of the most popular approaches to the IRL problem are Bayesian IRL \citep{ramachandran2007BIRL, rothkopf2011preference, choi2011MAP-BIRL} and Maximum Entropy IRL \citep{ziebart2008maximum, ho2016generative, finn2016guided}.
In this work, we extend both IRL formulations to demonstrations under varying environment dynamics. 
Note that this differs from the situation, where we observe demonstrations by experts of varying quality \citep{castro2019inverse}, or demonstrations by experts that optimise different rewards \citep{ramponi2020truly, likmeta2021dealing}, in a fixed environment. 
Moreover, \citet{cao2021identifiability} and \citet{rolland2022identifiability} study the identifiability of the true reward function in IRL and showed that when observing experts under different environment dynamics the true reward function can be identified up to a constant under certain conditions. 
However, it is important to note that in all of these cases, the learner is passive and does not actively seek information about the reward function by choosing specific experts or environments.





\vspace{0.05cm}
\paragraph{Active Inverse Reinforcement Learning.} 
The environment design problem that we consider in this paper can be viewed as one of active reward elicitation \citep{lopes2009active}. Prior work on active reward learning has focused on querying the expert for additional demonstrations in specific states \citep{lopes2009active, brown2018risk, lindner2021information, lindner2022active}, mainly with the goal of resolving the uncertainty that is due to the expert's policy not being specified accurately in these states. 
In contrast, we consider the situation where we cannot directly query the expert for additional information in specific states, but instead sequentially choose environments for the expert to act in.
Importantly, this means that the same state can be visited under different transition dynamics, which can be crucial to distinguish the true reward function among multiple plausible candidates \citep{cao2021identifiability, rolland2022identifiability}. 

In other related work, \citet{amin2017repeated} consider a repeated IRL setting in which the learner can choose \emph{any} task for the expert to complete (with full information of the expert policy).  
\citet{he2021assisted} study an iterative reward design setup where a human provides the learner with a proxy reward function, upon which the learner tries to choose an edge-case environment in which the proxy fails so that the human revises their proxy. 
In a similar vein, \citet{buning2022interactive} introduced Interactive IRL, where the learner interacts with a human in a collaborative Stackelberg game without knowledge of the joint reward function. 
This setting is similar to the framework presented in this paper in that the leader in a Stackelberg game can be viewed as designing environments by committing to specific policies.

\paragraph{Environment Design for Reinforcement Learning.}
Environment design and curriculum learning for RL aim to design a sequence of environments with increasing difficulty to improve the training of an autonomous agent~\citep{narvekar2020curriculum}. However, in contrast to our problem setup, observations in generated training environments are cheap, since this only involves actions from an autonomous agent, not a human expert.
As such, approaches like domain randomisation~\citep{tobin2017domain, akkaya2019solving} can be practical for RL, whereas they can be extremely inefficient and wasteful in an IRL setting. 
Moreover, in IRL we typically work with a handful of rounds only, so that slowly improving the environment generation process over thousands of training episodes (i.e., rounds) is impractical~\citep{dennis2020emergent, gur2021environment}. 
As a result, most methods, which are viable for the RL, can be expected to be unsuitable for the IRL problem.  

\section{Problem Formulation}\label{section:problem_formulation}
We now formally introduce the Environment Design for Inverse Reinforcement Learning framework. A Markov Decision Process (MDP) is a tuple $(\S, \A, \trans, \truereward, \gamma, \omega)$, where $\S$ is a set of states, $\A$ is a set of actions, $\trans: \S \times \A \times \S \to [0,1]$ is a transition function, $\truereward: \S \to \Reals$ is a reward function, $\gamma$ a discount factor, and $\omega$ an initial state distribution. 
We assume that there is a set transition functions $\Trans$ from which $\trans$ can be selected. Similar models have been considered for the RL problem under the name of Underspecified MDPs \citep{dennis2020emergent} or Configurable MDPs~\citep{metelli2018configurable, ramponi2021learning}. 


We assume that the true reward function, denoted $\truereward$, is unknown to the learner and consider the situation where the learner gets to interact with the human expert in a sequence of $m$ rounds.\footnote{Typically, expert demonstrations are a limited resource as they involve expensive human input. We thus consider a limited budget of $m$ expert trajectories that the learner is able to obtain.} 
More precisely, every round $k \in [m]$, the learner gets to select a demo environment $\trans_k \in \Trans$ for which an expert trajectory $\tau_k$ is observed. Our objective is to adaptively select a sequence of demo environments $\trans_1, \dots, \trans_m$ so as to recover a robust estimate of the unknown reward function. 
We describe the general framework for this interaction between learner and human expert in Framework~\ref{algorithm:general_framework_ed}. To summarise, a problem-instance in our setting is given by $(\S, \A, \Trans, \truereward, \gamma, \omega, m)$, where $\Trans$ is a set of environments, $\truereward$ is the \emph{unknown} reward function, and $m$ the learner's budget. 


\begin{algorithm}[t]
\floatname{algorithm}{Framework}
\caption{Environment Design for IRL}
\label{algorithm:general_framework_ed}
\begin{algorithmic}[1]
\STATE \textbf{input\,} set of environments $\Trans$, resources $m\in \Naturals$ 
\FOR{k = 1, \dots, m} 
\STATE Choose an environment $\trans_k \in \Trans$ 
\label{alg_line:environment_selection}
\STATE Observe expert trajectory $\tau_k$ in environment $\trans_k$
\STATE Estimate rewards from observations up to round $k$ 
\ENDFOR
\end{algorithmic}
\end{algorithm}

From Framework~\ref{algorithm:general_framework_ed} we see that the Environment Design for IRL problem has two main ingredients: a)~choosing useful demo environments for the human to demonstrate the task in (Section~\ref{section:environment_design}), and b)~inferring the reward function from expert demonstrations in multiple environments (Section~\ref{section:irl}).

\subsection{Preliminaries and Notation}
Throughout the paper, $\r$ denotes a generic reward function, whereas $\truereward$ refers to the true (unknown) reward function. We let $\Pi$ denote a generic policy space. 
Now, $\val_{\r, \trans}^\pi (s) \coloneqq \E [ \sum_{t=0}^\infty \gamma^t \r(s_t) \mid \pi, \trans, s_0 = s]$ is the expected discounted return, i.e., value function, of a policy $\pi$ under some reward function $\r$ and transition function $\trans$ in state $s$. For the value under the initial state distribution~$\omega$, we then merely write ${\val_{\r, \trans}^\pi \coloneqq \E_{s\sim \omega}[\val_{\r,\trans}^\pi (s)]}$ and denote its maximum by $\val_{\r, \trans}^* \coloneqq \max_{\pi} \val_{\r, \trans}^\pi$. We accordingly refer to the $\Q$-values under a policy $\pi$ by $\Q_{\r, \trans}^\pi(s, a)$ and their optimal values by $\Q_{\r, \trans}^*(s, a)$. 
In the following, we let $\pi^*_{\r, \trans}$ always denote the \emph{optimal policy} w.r.t.\ $\r$ and $\trans$, i.e., the policy maximising the expected discounted return in the MDP $(\S, \A, \trans, \r, \gamma, \omega)$.

In the following, we let $\tau$  denote expert trajectories. Note that in Framework~\ref{algorithm:general_framework_ed} every such trajectory is generated w.r.t.\ some transition dynamics $\trans$. In round $k$, we thus observe $\data_k = (\tau_k, \trans_k)$, i.e., the expert trajectory $\tau_k$ in environment $\trans_k$. We then write $\data_{1:k} = (\data_1, \dots, \data_k)$ for all observations up to (and including) the $k$-th round. Moreover, we let $\bel( \cdot \mid \data_{1:k})$ denote the posterior over reward functions given observations $\data_{1:k}$. For the prior $\P(\cdot)$, we introduce the convention that $\P(\cdot) = \P(\cdot \mid \data_{1:0})$.  

\section{Environment Design via Maximin Regret}\label{section:environment_design}
Our goal is to adaptively select demo environments for the expert based on our current belief about the reward function. 

In Section~\ref{subsection:minimax_br}, we introduce a maximin Bayesian regret objective for the environment design process which aims to select demo environments so as to ensure that our reward estimate is robust. 
Section~\ref{subsection:environment_generation} then deals with the selection of such environments when the set of environments exhibits a useful decomposable structure. We additionally provide a way to approximate the process when the set has an arbitrary structure or is challenging to construct.


\subsection{Maximin Bayesian Regret}\label{subsection:minimax_br}
We begin by reflecting on the potential \emph{loss} of an agent when deploying a policy $\pi$ under transition function $\trans$ and the \emph{true} reward function $\truereward$, given by the difference
\begin{equation*}
    \ell_{\truereward}(\trans, \pi) \coloneqq \val^*_{\truereward, \trans}  - \val_{\truereward, \trans}^{\pi}. \vspace{0.2cm}
\end{equation*}

The reward function $\truereward$ is unknown to us so that we can instead use our belief $\bel$ over reward functions\footnote{When we do not have a posterior over rewards, it is still possible to build a pseudo-belief upon point estimates. This approach is later explained in Section~\ref{subsection:maxent_irl}.}
and consider the \emph{Bayesian regret} (i.e., loss) of a policy $\pi$ under $\trans$ and $\bel$ given by
\begin{align*}
    \bayesregret_\bel (\trans, \pi) \coloneqq \E_{\r \sim \bel} \big[ \ell_\r (T, \pi)\big]. 
\end{align*} 
The concept of Bayesian regret is well-known from, e.g., online optimisation and online learning \cite{russo2014learning} and has been utilised for IRL in a slightly different form by~\citet{brown2018risk}. The idea is that given a (prior) belief about some parameter, we evaluate our policy against an oracle that knows the true parameter. 
Typically, under such uncertainty about the true parameter (in our case, reward function) we are interested in policies minimising the Bayesian regret: 
$$\min_{\pi \in \Pi} \bayesregret_\bel (\trans, \pi).$$ To derive an objective for the environment design problem, we then consider the maximin problem given by the worst-case environment $\trans$ for our current belief over reward functions $\bel$:
\footnote{We consider $\max_\trans \min_\pi$ and not the reverse, as we are interested in the maximin environment (and not minimax policy).}
\begin{equation}\label{eq:maximin_BR_objective}
    \max_{\trans \in \Trans} \min_{\pi \in \Pi} \bayesregret_\bel (\trans, \pi).
\end{equation}

What this means is that we search for an environment $\trans \in \Trans$ such that the regret-minimising policy w.r.t.\ $\bel$ performs the worst compared to the optimal policies w.r.t.\ the reward candidates $\r \sim \bel$. In other words, the maximin environment $\trans$ from \eqref{eq:maximin_BR_objective} can be viewed as the environment in which we expect our current reward estimate to perform the worst.  

Choosing environments for the expert according to \eqref{eq:maximin_BR_objective} also has the advantage that maximin environments are in general solvable for the expert, since the regret in degenerate or purely adversarial environments will be close to zero. Moreover, the regret objective is performance-based and not only uncertainty-based, such as entropy-based objectives \citep{lopes2009active}). This is typically desired as reducing our uncertainty about the rewards in states that are not relevant under any transition function in $\Trans$ (e.g., states that are not being visited by any optimal policy) is unnecessary and generally a wasteful use of our budget. 
Finally, we also see that if the Bayesian regret objective becomes zero, the posterior mean is guaranteed to be optimal in every demo environment. 

\begin{lemma}\label{lem:maximin_zero}
    If for some posterior $\bel(\cdot \mid \data)$ we have $\max_{\trans\in \Trans} \min_{\pi\in \Pi} \bayesregret_{\bel}(\trans, \pi) = 0$, then the posterior mean $\bar \r = \E_{\bel} [\r]$ is optimal for every $\trans \in \Trans$, i.e., $\bar \r$ induces an optimal policy in every environment contained in $\Trans$. 
\end{lemma} 

It is worth noting that our maximin Bayesian regret objective resembles several approaches to robust reinforcement learning, e.g., \cite{roy2017reinforcement, zhou2021finite, buening2023minimax, zhou2024natural}. However, it differs in that we are interested in the maximin environment (not minimax policy) and the Bayesian regret is defined w.r.t.\  a set of environments $\Trans$ \emph{and} a distribution over reward functions.

\subsection{Finding Maximin Environments}\label{subsection:environment_generation} 
\paragraph{Structured Environments.} 
Often the set of environments has a useful structure that can be exploited to search the space of environments efficiently. We here consider the special case where each environment $\trans \in \Trans$ is build from a collection of transition matrices $\trans_s$. Similar setups can be found in the robust dynamic programming literature (e.g., \cite{iyengar2005robust, nilim2005robust, xu2010distributionally, mannor2016robust}).   


Let $\trans_s \in \Reals^{\S \times A}$ denote a state-transition matrix dictating the transition probabilities in state $s$. 
We can identify any transition function $\trans$ with a family of state-transition matrices $\{\trans_s\}_{s\in \S}$. We then say that an environment set $\Trans$ allows us to make \emph{state-individual transition choices} if there exist sets $\Trans_s$ such that $\Trans = \{ \{ \trans_s\}_{s \in \S} \colon \trans_s \in \Trans_s\}$. In other words, we can choose a new environment $\trans$ by arbitrarily combining transition matrices for each state. Note that this of course allows for the case where the transitions in some state $s$ are fixed, i.e., $\Trans_s = \{\trans_s\}$. When we can make such state-individual transition choices, we can use an extended value iteration approach to approximate the maximin environment efficiently. The extended value iteration algorithm is specified in Appendix~\ref{appendix:algorithms}, Algorithm~\ref{algorithm:extended_VI}. 

\paragraph{Arbitrary Environments.} In some situations, the set of environments $\Trans$ may not exhibit any useful structure. Moreover, we may not even have explicit knowledge of the transition functions in $\Trans$, but can only access a set of corresponding simulators. In this case, we are left with approximating the maximin environment \eqref{eq:maximin_BR_objective} by sampling simulators from $\Trans$ and performing policy rollouts. We describe the complete procedure in Appendix~\ref{appendix:algorithms}, Algorithm~\ref{algorithm:brute_force_ed}.

\paragraph{Flexible Environment Set Construction.}

Although our assumption initially considers the simplest scenario where any environment within $\Trans$ is accessible at any time, this may become impractical when the process of building environments is labour-intensive. Instead of probing the entire set $\Trans$ at each environment design step, our framework allows some flexibility. With little approximation in Framework~\ref{algorithm:general_framework_ed} Line~\ref{alg_line:environment_selection}, we can select the next environment from a new subset $\Trans_\subset \subset \Trans$. This allows users to progressively build new environments with every additionally desired round.



\section{Inverse Reinforcement Learning with Multiple Environments}\label{section:irl}
We now explain how to learn about the reward function from demonstrations that were provided under multiple environment dynamics. To do so, we will extend Bayesian and MaxEnt IRL methods to this setting, and combine them with environment design to obtain the $\edbirl$ and $\edairl$ algorithm, respectively. While $\edbirl$ is designed for simple tabular problems due to its high complexity, $\edairl$ can be applied to environments with large or continuous action/observation spaces.

\begin{algorithm}[t]
\caption{\texttt{ED-BIRL}: Environment Design for \birl{}}
\label{algorithm:ED-BIRL}
\begin{algorithmic}[1]
\STATE \textbf{input\,} environments $\Trans$, prior $\P$, budget $m \in \mathbb{N}$
\FOR{$k = 1, \dots, m$}
\STATE Sample rewards from $\P(\cdot \mid \data_{1:{k-1}})$ using \birl{} 
\STATE Construct empirical distribution $\empP_{k-1}$ from samples\label{alg_line:emp_distr}
\STATE Find $\trans_k \in \argmax_{\trans} \min_\pi \bayesregret_{\empP_{k-1}}(\trans, \pi)$ 
\STATE Observe trajectory $\tau_{k}$ in $\trans_{k}$, i.e., $\data_k = (\tau_k, \trans_k)$
\ENDFOR
\RETURN $\texttt{BIRL}(\data_{1:m})$
\end{algorithmic}
\end{algorithm} 

\subsection{The Bayesian Setting: \edbirl{}} \label{subsection:birl}
The Bayesian perspective to the IRL problem provides a principled way to reason about reward uncertainty~\cite{ramachandran2007BIRL}.  Typically, the human is modelled by a Boltzmann-rational policy~\cite{jeon2020reward}. This means that for a given reward function $\r$ and transition function $\trans$ the expert is acting according to a softmax policy
$\pi^{\textrm{softmax}}_{\r, \trans} (a \mid s) = \frac{\exp (c \Q^*_{\r, \trans}(s, a))}{\sum_{a'} \exp(c \Q^*_{\r, \trans}(s, a'))}$,
where the parameter $c$ relates to our judgement of the expert's optimality.\footnote{Note that when using MCMC Bayesian IRL methods we can also perform inference over the parameter $c$ and must not assume knowledge of the expert's optimality.}
Given a prior distribution~$\P(\cdot)$, the goal of Bayesian IRL is to recover the posterior distribution $\P(\cdot \mid \data)$ and to either sample from the posterior using MCMC~\cite{ramachandran2007BIRL, rothkopf2011preference} or perform MAP estimation~\cite{choi2011MAP-BIRL}. 

In our case, the data is given by the sequence $\data_{1:k} = (\data_1, \dots, \data_k)$ with $\data_i = (\tau_i, \trans_i)$. We see that this is no obstacle as the likelihood factorises as 
$\P(\data_{1:k} \mid \r) = \prod_{i \leq k} \P( \tau_i \mid \r, \trans_i)$,
since the expert trajectories (i.e., expert policies) are conditionally independent given the reward function and transition function. The likelihood of each expert demonstration is then given by $\P(\tau_i \mid \r, \trans_i) = \prod_{(s,a)\in \tau_i} \pi^{\textrm{softmax}}_{\r, \trans_i}(a \mid s)$.
Hence, the reward posterior can be expressed as
\begin{equation}\label{eq:birl_posterior}
    \P(\r \mid \data_{1:k}) \propto \prod_{i \leq k}\prod_{(s,a)\in \tau_i} \pi^{\textrm{softmax}}_{\r, \trans_i}(a \mid s)  \cdot \P(\r).
\end{equation}
As a result, we can, for instance, sample from the posterior using the Policy-Walk algorithm from~\cite{ramachandran2007BIRL} with minor modifications or the Metropolis-Hastings Simplex-Walk algorithm from~\cite{buning2022interactive}. We generally denote any Bayesian IRL algorithm that is capable of sampling from the posterior by $\birl$. 

Plugging $\birl$ into our environment design framework, we get the $\edbirl$ procedure detailed in Algorithm~\ref{algorithm:ED-BIRL}. Note that, in practice, we approximate the posterior $\P(\cdot \mid \data_{1:k})$ by sampling rewards and constructing an empirical distribution~$\hat \P_{k}$.


\subsection{The MaxEnt Setting: \airlme{} and \edairl{}}\label{subsection:maxent_irl}

In the following, we describe how to extend MaxEnt IRL methods to demonstrations from multiple environments, and use them in combination with environment design for IRL. 









\begin{algorithm}[t]
\caption{$\edairl$: Environment Design for $\airl$}
\label{algorithm:ed-airl}
\begin{algorithmic}[1]
\STATE \textbf{input\,} environments $\Trans$, budget $m \in \mathbb{N}$ 
\FOR{$k = 1, \dots, m$}
\IF{$k = 1$}
\STATE Choose $\trans_k$ arbitrarily from $\Trans$
\ELSE
\STATE Let $\hat \P_{k-1} \equiv \mathcal{U}(\{\r_1, \dots, \r_{k-1}\})$
\STATE Find $\trans_{k} \in \argmax_{\trans}\bayesregret_{\empP_{k-1}}(\trans, \pi^*_{\r_{1:k-1}, \trans})$ \hfill \label{algorithm:ed-airl:environment_selection} 
\ENDIF 
\STATE Observe trajectory $\tau_k$ in $\trans$, i.e., $\data_k = (\tau_k, \trans_k)$
\STATE Compute point estimate $\r_k = \airl(\data_k)$
\STATE Compute best guess $\r_{1:k} = \airlme(\data_{1:{k}})$ 

\ENDFOR

\RETURN $\airlme(\data_{1:m})$
\end{algorithmic}
\end{algorithm}

\paragraph{Reward Inference with Multiple Environments.}
In MaxEnt IRL, the reward function is assumed to be parameterised by some parameter $\theta$. Given observations $\data_{1:k} = (\data_1, \dots, \data_k)$ with $\data_i = (\tau_i, \trans_i)$, our goal for multiple environments is to solve the maximum likelihood problem
\begin{align}\label{equation:maxent-objective}
    \argmax_\theta \sum_{(\tau, \trans) \in \data_{1:k}} \log \P(\tau \mid \theta, \trans),
\end{align}
where we again used that trajectories are independent conditional on the reward parameter $\theta$ and the dynamics $\trans$. 
Consequently, the only difference to the original MaxEnt IRL formulation is that we now sum over pairs $(\tau, \trans)$ instead of just trajectories $\tau$. 
The specific algorithm we consider here is Adversarial IRL ($\airl$) \citep{fu2017learning}, which 
frames the optimisation of \eqref{equation:maxent-objective} as a generative adversarial network. 

To extend $\airl$ to multiple environments, we can consider $k$ distinct policies $\pi_1, \dots, \pi_k$ used to generate trajectories in environments $\trans_1, \dots, \trans_k$, respectively, and discriminators $D_{1, \theta, {\phi_1}}, \dots, D_{k, \theta, {\phi_k}}$ given by 
$$D_{i, \theta, {\phi_i}}(s, a, s') =  \frac{\exp(f_{\theta, {\phi_i}}(s, s'))}{\exp(f_{\theta, {\phi_i}}(s, s')) + \pi_i (a\mid s)},$$
with $f_{\theta, {\phi_i}} (s , s') = g_\theta (s) + \gamma h_{\phi_i}(s') - h_{{\phi_i}}(s)$. 
Here, $g_\theta$ is the state-only reward approximator, and $h_{\phi_i}$ a shaping term specific to each environment $\trans_i$. We refer to this algorithm as \airlme{} and defer for a detailed description to Algorithm~\ref{algorithm:adv_maxent_irl} in Appendix~\ref{appendix:algorithms}. 

\paragraph{Environment Design for AIRL: \edairl{}.} 
This algorithm also relies on Bayesian regret, even though we do not have an analytical posterior as in the Bayesian setting. We substitute the posterior with a uniform distribution $\mathcal{U}(\{\r_1, \dots, \r_k\})$ over point estimates $\r_i$, each separately calculated with standard \airl{} from demonstration $\data_i$. 
We then also replace the posterior mean $\bar \r$ with the multi-environment estimate obtained from $\airlme{}(\data_{1:k})$, denoted $\r_{1:k}$.\footnote{Note that $\r_{1:k}$ is generally not equal to the mean of $\mathcal{U}(\{\r_1, \dots, \r_k\})$.} Here, $\r_{1:k}$ represents our best guess about the reward function after round $k$.  
In particular, for a given environment $\trans$, the optimal policy $\pi^*_{\r_{1:k}, \trans}$ w.r.t.\ $\r_{1:k}$ and $\trans$ is the policy we expect to perform best in environment $\trans$. In round $k+1$, we then choose the environment $T$ in which the reward function $\r_{1:k}$ performs the worst compared to the reward functions $\r_{1}, \dots, \r_k$. This environment is given by $\argmax_{\trans \in \Trans} \bayesregret_{\mathcal{U}(\{\r_1, \dots, \r_k\})} (\trans, \pi^*_{\r_{1:k}, \trans})$. 
The \edairl{} procedure is detailed in Algorithm~\ref{algorithm:ed-airl}.

\begin{remark}
    Note that, in every round $k$, we can of course let \edbirl{} and \edairl{} query a batch of several demonstrations from the environment $\trans_k$. 
    In the case of \edairl{}, our experiments showed that this imporves the stability of the reward inference. 
\end{remark}

\section{Experiments}\label{section:experiments}
The primary goal of our experiments is to address the following two questions:
(1) Can we \emph{recover the true reward function} by adaptively designing environments?
(2) Are the reward estimates \emph{more robust} by doing so?

To answer the first question, Section~\ref{subsection:exp_maze} considers the maze task already introduced in Figure~\ref{fig:intro} of the introduction, as it allows us to nicely visualise the learned reward functions. 
To answer the second question, we evaluate our environment design approach on continuous control tasks in Section~\ref{subsection:exp_continuous}, where we slightly perturb environment dynamics to test the robustness of learned rewards. We also include an ablation study by replacing our maximin Bayesian regret environment design with domain randomisation. This allows us to separate the effect of active environment design from the effect of learning a reward function from multiple (possibly randomly chosen) environments. Additional experimental results and details are provided in Appendix~\ref{appendix:experiment_results} and Appendix~\ref{appendix:experiment_details}.

\begin{figure*}[t]

     \centering
     \begin{subfigure}[b]{0.32\textwidth}
         \centering
         \begin{tikzpicture}
            \draw (0, 0) node[inner sep=0] {
            \includegraphics[width=0.45\textwidth]{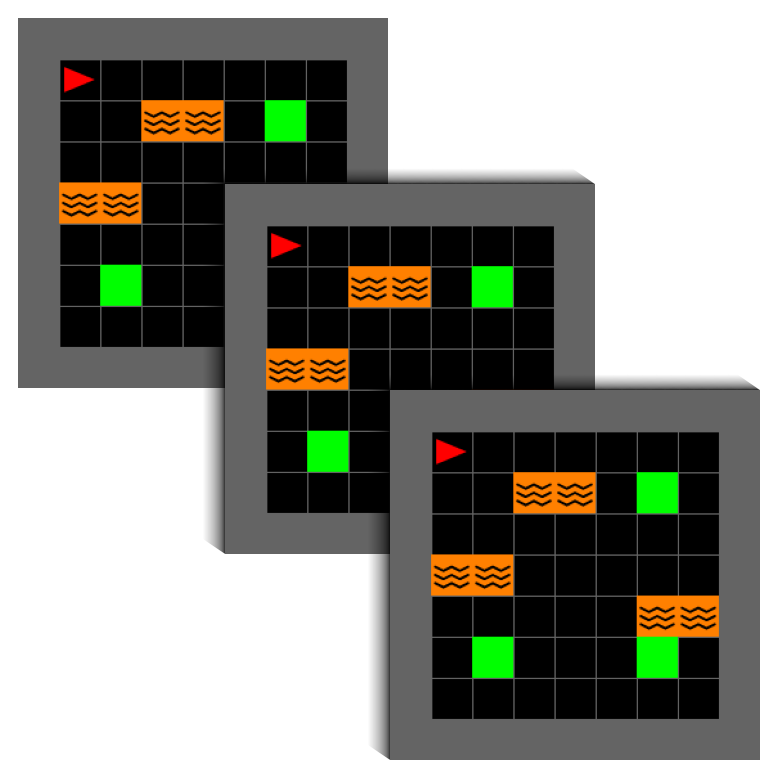}};
            \draw (0, 1.35) node {\tiny Selected Mazes};
        \end{tikzpicture}
         \begin{tikzpicture}
            \draw (0, 0) node[inner sep=0] {
            \includegraphics[width=0.45\textwidth]{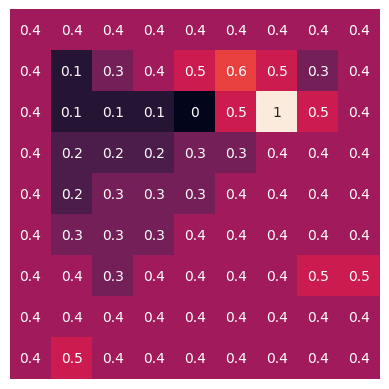}};
            \draw (0, 1.35) node {\tiny Final Reward Estimate};
        \end{tikzpicture}
         \caption{\texttt{Fixed Environment\, BIRL}}
                  \label{subfig:basic_fixed_3rd_round}

     \end{subfigure}
           \hfill
    \begin{subfigure}[b]{0.32\textwidth}
         \centering
         \begin{tikzpicture}
            \draw (0, 0) node[inner sep=0] {
            \includegraphics[width=0.45\textwidth]{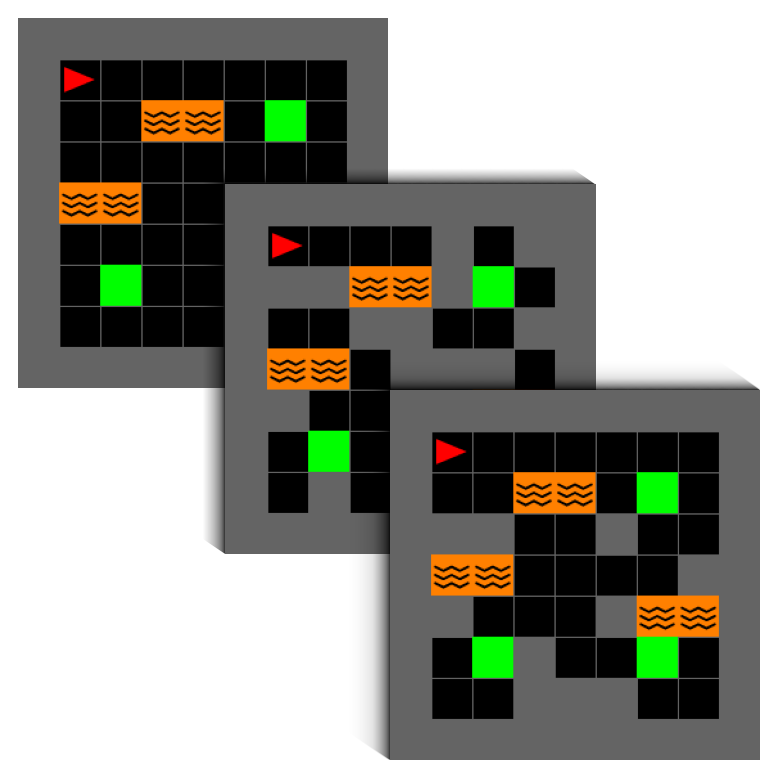}};
            \draw (0, 1.35) node {\tiny Selected Mazes};
        \end{tikzpicture}
         \begin{tikzpicture}
            \draw (0, 0) node[inner sep=0] {
            \includegraphics[width=0.45\textwidth]{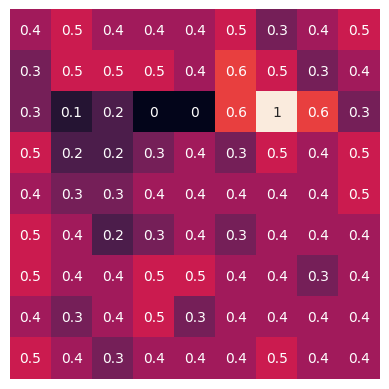}};
            \draw (0, 1.35) node {\tiny Final Reward Estimate};
        \end{tikzpicture}
         \caption{\texttt{Domain Randomisation BIRL}}
                  \label{subfig:basic_dom_ran_3rd_round}
    \end{subfigure}
    \hfill
     \begin{subfigure}[b]{0.32\textwidth}
         \centering
         \begin{tikzpicture}
            \draw (0, 0) node[inner sep=0] {\includegraphics[width=0.45\textwidth]{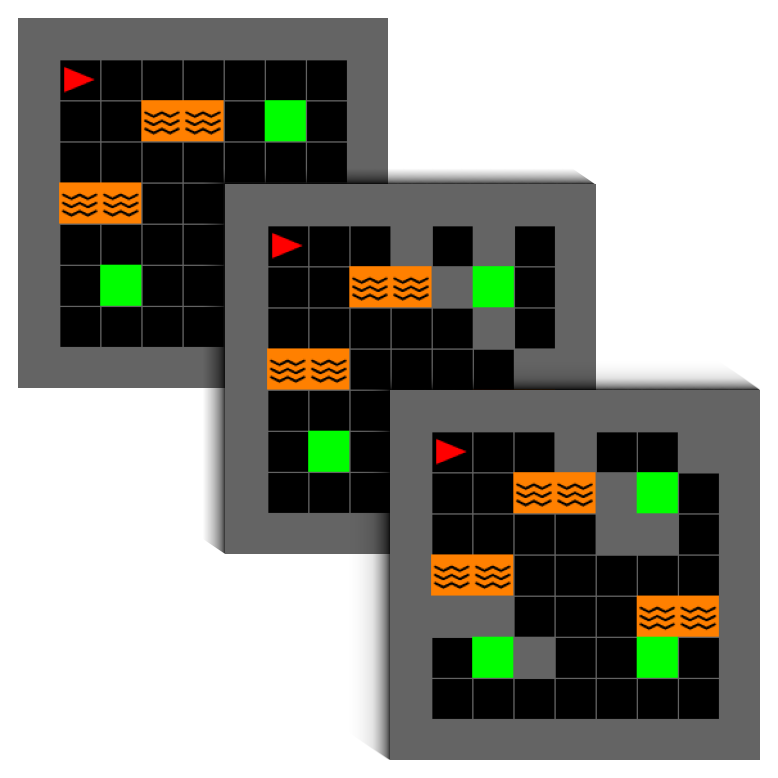}};
            \draw (0, 1.35) node {\tiny Selected Mazes};
        \end{tikzpicture}
         \begin{tikzpicture}
            \draw (0, 0) node[inner sep=0] {
            \includegraphics[width=0.45\textwidth]{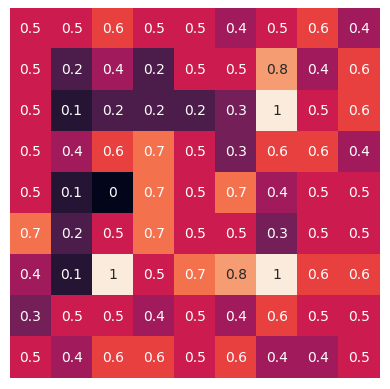}};
            \draw (0, 1.35) node {\tiny Final Reward Estimate};
        \end{tikzpicture}
         \caption{\texttt{ED-BIRL}}
         \label{subfig:basic_ed_birl_3rd_round}
     \end{subfigure}
     

     \caption{The discrete maze task from Figure~\ref{fig:intro}. The goal in the discrete maze environment is to reach one of the three green goal squares while avoiding lava squares. For each approach, we show on the left the chosen mazes and visualise on the right the posterior mean reward after three rounds. In (a), the expert always acts in the same, fixed maze. In (b), the maze is randomly generated by adding obstacles, i.e., gray squares, uniformly at random. The proposed \edbirl\, approach, which adaptively chooses maze layouts based on past reward estimates, is shown in (c). We use the same colour scale as in Figure~\ref{fig:intro}, which ranges from black (0.0) to red (0.5) to white (1.0).}
     \label{fig:maze_exp_3rd_round}
\end{figure*}
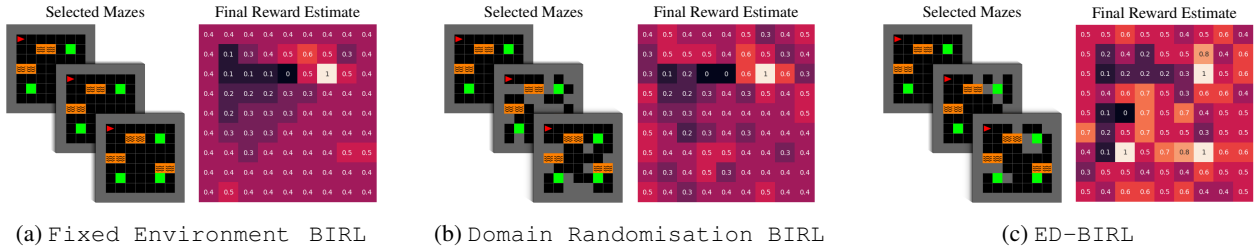

\subsection{Recovering the True Reward Function}\label{subsection:exp_maze}
In this experiment, we consider a discrete maze task inspired by minigrid \citep{chevalier2023minigrid}, in which the learner has the ability to add obstacles to a base layout of the maze. We visualise the selected mazes and estimated rewards, and evaluate whether our approach can recover the true reward function. 

\paragraph{Experimental Setup.} In the maze task, the objective is to reach one of three goal squares while avoiding lava.
Here, any maze from the set of environments is obtained by adding, removing, or moving gray obstacles. Note that the goal and lava squares cannot be moved. The learner gets to select mazes from this set and observes two expert trajectories for every chosen maze.\footnote{We give the learner two trajectories per chosen environment to provide a stronger learning signal to $\texttt{BIRL}$.}
The true reward function, which is unknown to the learner, yields reward $1$ in goal squares and reward $-1$ in lava squares. We discount the problem by $\gamma = 0.9$, so that the expert tries to reach a goal square as quickly as possible. 

We compare our approach, \texttt{ED-BIRL}, with learning from a single, fixed maze as well as learning from mazes that were randomly generated. We randomly generate these mazes by adding an obstacle to a tile with probability $0.3$.\footnote{Naturally, such randomly generated mazes can be very different every iteration, and we can only display exemplary mazes for domain randomisation in Figure~\ref{fig:maze_exp_3rd_round}. Nevertheless, the presented examples shall serve as an illustration of the disadvantages of using domain randomisation in IRL, where demonstrations are costly.} The inference for all three approaches is done using $\texttt{BIRL}$ and the computed reward estimates are scaled to $[0,1]$ and rounded to one decimal.

\vspace{-0.2cm}
\paragraph{Results.} 
In Figure~\ref{fig:maze_exp_3rd_round}, we observe that \texttt{ED-BIRL} recovers the location of all three goal squares after three rounds. 
Moreover, the learner is able to identify the location of all lava strips in Figure~\ref{subfig:basic_ed_birl_3rd_round}, i.e., squares with negative reward. 
Hence, by adaptively designing a sequence of environments, \texttt{ED-BIRL} is capable of recovering all performance-relevant aspects of the unknown reward function. 

In contrast, learning from a fixed environment (Figure~\ref{subfig:basic_fixed_3rd_round}) as well as domain randomisation (Figure~\ref{subfig:basic_dom_ran_3rd_round}) both fail to recover the location of all goal squares, let alone lava squares. In a fixed maze, any near-optimal policy will visit the closest goal state only, which in this case is in the top right corner of the maze. 
We also see that using domain randomisation is impractical for IRL, as we require carefully curated mazes to recover the true reward function. Even worse, by obliviously randomising the maze layout, we may create unsolvable environments for the human expert, which yield no information at all (see Figure~\ref{subfig:basic_dom_ran_3rd_round}). Additional results on the tabular maze task can be found in Appendix~\ref{appendix:experiment_results}.

\begin{table*}[t]
\setlength{\tabcolsep}{2pt} 
\renewcommand{\arraystretch}{1.3} 
\resizebox{\linewidth}{!}{%
\begin{tabular}{L{0.1\textwidth}||C{0.05625\textwidth}|C{0.05625\textwidth}||C{0.05625\textwidth}|C{0.05625\textwidth}||C{0.05625\textwidth}|C{0.05625\textwidth}||C{0.05625\textwidth}|C{0.05625\textwidth}||}
\multicolumn{1}{C{0.1\textwidth}||}{}&
  \multicolumn{2}{C{0.225\textwidth}||}{Continuous maze} &
  \multicolumn{2}{C{0.225\textwidth}||}{Hopper} &
  \multicolumn{2}{C{0.225\textwidth}||}{HalfCheetah} &
  \multicolumn{2}{C{0.225\textwidth}||}{Swimmer} \\ \cline{2-9}
\multicolumn{1}{C{0.1\textwidth}||}{} &
  \multicolumn{1}{C{0.1125\textwidth}|}{Demo} &
  \multicolumn{1}{C{0.1125\textwidth}||}{Test} &
  \multicolumn{1}{C{0.1125\textwidth}|}{Demo} &
  \multicolumn{1}{C{0.1125\textwidth}||}{Test} &
  \multicolumn{1}{C{0.1125\textwidth}|}{Demo} &
  \multicolumn{1}{C{0.1125\textwidth}||}{Test} &
  \multicolumn{1}{C{0.1125\textwidth}|}{Demo} &
  \multicolumn{1}{C{0.1125\textwidth}||}{Test} \\ \midrule
\multicolumn{1}{L{0.1\textwidth}||}{$\edairl$} &
  \multicolumn{1}{C{0.115\textwidth}|}{\textbf{68±04}} &
  \multicolumn{1}{C{0.115\textwidth}||}{\textbf{71±02}} &

  \multicolumn{1}{C{0.115\textwidth}|}{\textbf{63±07}} &
  \multicolumn{1}{C{0.115\textwidth}||}{52±04} &

  \multicolumn{1}{C{0.115\textwidth}|}{\textbf{40±11}} &
  \multicolumn{1}{C{0.115\textwidth}||}{35±13} &

  \multicolumn{1}{C{0.115\textwidth}|}{80±19} &
  \multicolumn{1}{C{0.115\textwidth}||}{69±12} \\
  
\multicolumn{1}{L{0.1\textwidth}||}{$\drairl$} &
  \multicolumn{1}{C{0.115\textwidth}|}{52±07} &
  \multicolumn{1}{C{0.115\textwidth}||}{53±12} &
  
  \multicolumn{1}{C{0.115\textwidth}|}{59±06} &
  \multicolumn{1}{C{0.115\textwidth}||}{\textbf{56±04}} &

  \multicolumn{1}{C{0.115\textwidth}|}{\textbf{40±13}} &
  \multicolumn{1}{C{0.115\textwidth}||}{\textbf{40±11}} &
    
  \multicolumn{1}{C{0.115\textwidth}|}{45±04} &
  \multicolumn{1}{C{0.115\textwidth}||}{53±05} \\ \cdashline{1-9}
\multicolumn{1}{L{0.1\textwidth}||}{$\airl$} &
  \multicolumn{1}{C{0.115\textwidth}|}{33±09} &
  \multicolumn{1}{C{0.115\textwidth}||}{52±07} &
  
  \multicolumn{1}{C{0.115\textwidth}|}{38±03} &
  \multicolumn{1}{C{0.115\textwidth}||}{34±04} &

  \multicolumn{1}{C{0.115\textwidth}|}{29±09} &
  \multicolumn{1}{C{0.115\textwidth}||}{16±07} &
  
  \multicolumn{1}{C{0.115\textwidth}|}{40±09} &
  \multicolumn{1}{C{0.115\textwidth}||}{44±08} \\ 
\multicolumn{1}{L{0.1\textwidth}||}{$\rime$} &
  \multicolumn{1}{C{0.115\textwidth}|}{-105±12} &
  \multicolumn{1}{C{0.115\textwidth}||}{-52±03} &
  
  \multicolumn{1}{C{0.115\textwidth}|}{61±01} &
  \multicolumn{1}{C{0.115\textwidth}||}{53±02} &

  \multicolumn{1}{C{0.115\textwidth}|}{-21±08} &
  \multicolumn{1}{C{0.115\textwidth}||}{-11±09} &

  \multicolumn{1}{C{0.115\textwidth}|}{-05±01} &
  \multicolumn{1}{C{0.115\textwidth}||}{-04±01} \\
  
\multicolumn{1}{L{0.1\textwidth}||}{$\gail$} &
  \multicolumn{1}{C{0.115\textwidth}|}{20±05} &
  \multicolumn{1}{C{0.115\textwidth}||}{17±01} &
  
  \multicolumn{1}{C{0.115\textwidth}|}{40±02} &
  \multicolumn{1}{C{0.115\textwidth}||}{34±01} &
  
  \multicolumn{1}{C{0.115\textwidth}|}{-12±02} &
  \multicolumn{1}{C{0.115\textwidth}||}{-06±01} &

  \multicolumn{1}{C{0.115\textwidth}|}{111±00} &
  \multicolumn{1}{C{0.115\textwidth}||}{110±01} \\ 
\multicolumn{1}{L{0.1\textwidth}||}{$\bc$} &
  \multicolumn{1}{C{0.115\textwidth}|}{11±00} &
  \multicolumn{1}{C{0.115\textwidth}||}{22±00} &
  
  \multicolumn{1}{C{0.115\textwidth}|}{-12±02} &
  \multicolumn{1}{C{0.115\textwidth}||}{-06±01} &
  
  \multicolumn{1}{C{0.115\textwidth}|}{-23±01} &
  \multicolumn{1}{C{0.115\textwidth}||}{-14±01} &

  \multicolumn{1}{C{0.115\textwidth}|}{\textbf{124±01}} &
  \multicolumn{1}{C{0.115\textwidth}||}{\textbf{130±01}}
\end{tabular}%
}
\vspace{3px}
\caption{\label{tab:normalized_returns} Performance on the continuous control tasks. Average normalised returns and their standard error over the demo and test sets (averaged over 10 runs for MuJoCo and 5 runs for the continuous maze). 100 indicates expert performance, whereas 0 indicates the performance of a random policy.}
\vspace{-0.1cm}
\end{table*}

\subsection{Learning Robust Reward Functions in Continuous Control Problems}\label{subsection:exp_continuous}

In this set of experiments, we evaluate \edairl{} on continuous control tasks. We consider a continuous maze task \citep{fu2017learning} as well as MuJoCo environments \citep{todorov2012mujoco}. 
For these test suites, we no longer have direct access to the transition function but can only access their corresponding simulators. We here compare $\edairl$ with Domain Randomisation \airl{} ($\drairl$) and standard $\airl$ \citep{fu2017learning}. Moreover, we also benchmark against imitations learning approaches, including Behavioral Cloning ($\bc$), Generative Adversarial Imitation Learning ($\gail$) \citep{ho2016generative} and $\rime$ \cite{chae2022robust}. Notably, $\rime$ is an imitation learning algorithm that tries to find a robust policy based on demonstrations from multiple environments. 

\paragraph{Experimental Setup.}
We treat the continuous maze problem without additional walls as well as the original MuJoCo environments in their default settings as \emph{base environments}. For each task, we then build a \emph{disjoint} set of \emph{demo environments} (i.e., environments in which we could potentially observe the expert) and \emph{test environments} by altering their respective base environments. 

In the continuous maze, this is done by adding walls to the maze. For the MuJoCo environments, we tweak the environment dynamics by changing parameters of the physics engine, which include body features of the agent as well as gravity parameters (see Appendix~\ref{appendix:experiment_details} for details). In total, we create 20 demo environments and 10 test environments for each task, where we want to emphasize that these sets are disjoint so that we never observe the expert in any test environments. We include the base environment by default in the set of demo environments. Exemplary demo and test environments are illustrated in Figure~\ref{fig:environment_tweaking}.

For the approaches that learn from multiple environments ($\edairl$, $\drairl$, $\rime$), the learner has access to the set of demo environments out of which it can sequentially choose $5$ ($10$ for the continuous maze) demo environments for the expert to act in. The learner always chooses the base environment first, serving as the initialisation for $\edairl$ and ensuring a fair comparison.
Subsequently, $\drairl$ chooses environments from the demo set uniformly at random without replacement, and then uses \airlme{} to compute a reward estimate from all observed trajectories. As $\rime$ does not estimate the reward function, we cannot plug-in our environment design approach. Instead, we select environments from the demo set analogously to $\drairl$.


In total, every approach observes a total of $m=50$ expert trajectories. Specifically, $\edairl$, $\drairl$, and $\rime$ observe $10$ trajectories ($5$ for the continuous maze) in each environment they select. In contrast $\airl$, $\gail$, and $\bc$ observe $50$ trajectories in the base environment. 

For the IRL approaches ($\edairl$, $\drairl$, $\airl$), we evaluate the final reward estimates after observing all 50 expert trajectories. We do this by optimising a fresh policy in each of the demo and test environments using the learned reward function. We then evaluate this policy under the true reward function in each environment before reporting the average return across the demo and test sets. For the imitation learning algorithms ($\rime$, $\gail$, $\bc$) we report the average return the imitating policy achieves when deployed on the demo and test set environments.

\begin{figure}[t]
    \vspace{.4cm}
     \centering
         \begin{subfigure}[b]{0.23\textwidth}
         \centering
         \begin{subfigure}[b]{0.48\textwidth}
            \centering
            \includegraphics[width=\textwidth]{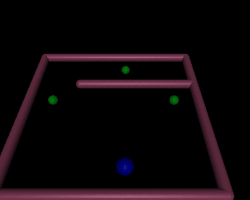}\hfill 
            \caption*{Demo}
         \end{subfigure}
         \begin{subfigure}[b]{0.48\textwidth}
            \centering
            \includegraphics[width=\textwidth]{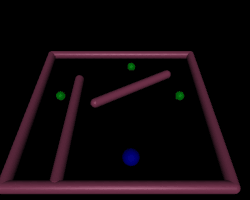}\hfill 
            \caption*{Test}
         \end{subfigure}
         \caption{Continuous maze\label{fig:maze_sample}}
         \end{subfigure}
         \begin{subfigure}[b]{0.23\textwidth}
         \centering
         \begin{subfigure}[b]{0.48\textwidth}
            \centering
            \includegraphics[width=\textwidth]{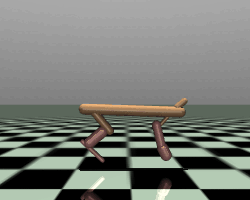}\hfill 
            \caption*{Demo}
         \end{subfigure}
         \begin{subfigure}[b]{0.48\textwidth}
            \centering
            \includegraphics[width=\textwidth]{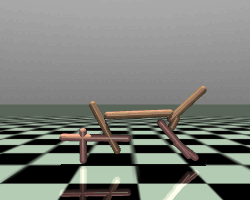}\hfill 
            \caption*{Test}
         \end{subfigure}
         \caption{HalfCheetah\label{fig:hopper_environments}}
         \end{subfigure}

         

        \caption{Examples of demo and test environments.
        \label{fig:environment_tweaking}}

        \vspace{-0.3cm}

\end{figure}

\paragraph{Results.} 
In Table~\ref{tab:normalized_returns}, we see that $\edairl$ overall achieves the highest performance on the demo and test sets. 
In most cases, $\edairl$ outperforms $\drairl$, sometimes significantly (Swimmer), with the exception of HalfCheetah, where both approaches perform similarly. 
Notably, $\edairl$ consistently outperforms $\airl$ by a large margin in all test suites. 

As expected, the imitating policies from $\gail$ and $\bc$ fail to transfer across variations of the base environment (Table~\ref{tab:normalized_returns}). An exemption is the Swimmer environment, in which the imitating policies of $\gail$ and $\bc$ were extremely robust to any changes in the environment dynamics (gravity and body features). 
Lastly, in our experimental setup, $\rime$ struggles to learn a reasonable policy with the exception of Hopper where $\rime$ performs similarly to $\edairl$ and $\drairl$. The reason for its bad performance could be that in our case the environments in the demo and test set are much more different from one another compared to the study done in \citet{chae2022robust}. 

In Figure~\ref{fig:airl_vs_edairl_budget}, we also analyse the improvement of $\edairl$ versus $\airl$ as we increase the total number of expert trajectories in the continuous maze task. Whereas $\airl$ does not significantly improve with more trajectories, we observe that $\edairl$ shows steady improvement until plateauing at around 65\% expert level performance when given a budget of $25$ trajectories. 
More generally, even though $\airl$ is designed to learn a robust reward function, it appears that learning from a single environment is insufficient to generalise to slight variations of a base environments, independently of the number of trajectories. 

We provide several additional experimental results for $\edbirl$ and $\edairl$ in Appendix~\ref{appendix:experiment_results}, including insights into the environments that are being selected by $\edairl$ in the continuous maze task. 

\begin{figure}[t]
     \centering
         \centering
         \includegraphics[width=0.48\textwidth]{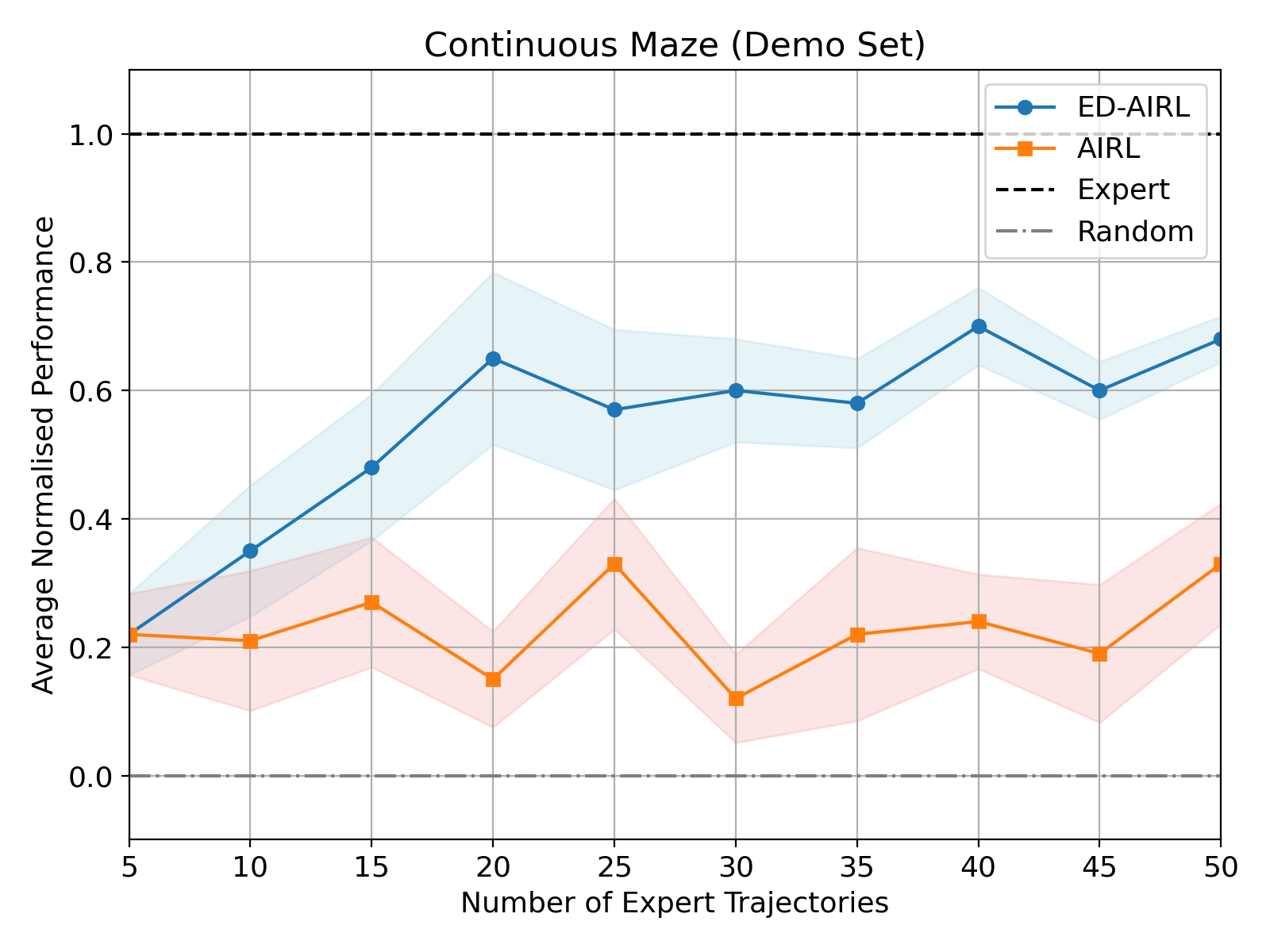}
         \vspace{-.6cm}
        \caption{Average normalised performance on the demo set in the continuous maze as we increase the number of expert trajectories (averaged over 5 runs). The standard error is shown in shaded colour. Every 5 trajectories, $\edairl$ chooses a new environment for the expert to act in. In contrast, $\airl$ always observes the expert act in the same base environment. \label{fig:airl_vs_edairl_budget}}
\end{figure}

\vspace{-0.1cm}
\section{Conclusion} 
\label{section:conclusion}

We introduced Environment Design for Inverse Reinforcement Learning, which concerns how to optimally design a sequence of environments to acquire expert demonstrations. We proposed a general methodology for doing so, based on a maximin Bayesian regret objective. This idea can be directly combined with Bayesian IRL methods to obtain the $\edbirl$ algorithm. We also propose an approximate method, \edairl{}, based on an ensemble of point estimates from the $\airl$ procedure. Experimentally, we show that our approach is able to recover almost all performance-relevant aspects of the unknown reward function, and improves the robustness of learned reward functions. 

\vspace{-0.2cm}
\paragraph{Limitations.} The computational load of $\edairl$ can become heavy since the number of policy optimisations grows quadratically with the number of rounds. 
We can reduce the computational load by sampling a subset $\mathcal{T}_\subset$ of environments for the environment design process. In this case, $\edairl$ requires $\mathcal{O}(|\mathcal{T}_\subset|m^2)$ policy optimisations, where $m$ is the number of times we select an environment.  


\vspace{-0.2cm}
\paragraph{Future directions.} In our experiments, we found the inference of $\airl$, as well as other MaxEnt IRL methods, to be unstable. 
In future work, we could use a Gaussian process model for the reward function~\cite{levine2011nonlinear}. This can be made scalable, for instance, through a variational inference approach~\cite{chan2021scalable}. 

\bibliography{ref}
\bibliographystyle{icml2024}

\newpage
\appendix
\onecolumn

\newpage 

\section*{Appendix}

\section{Proofs}

\begin{proof}[Proof of Lemma~\ref{lem:maximin_zero}]
For simplicity of exposition, we assume here that the posterior $\P$ is discrete. Now, as the value function is linear in rewards, we have 
\begin{align*}
    \min_{\pi} \bayesregret_\bel (\trans, \pi) = \bayesregret_\bel (\trans, \pi_{\bar \r, \trans}),
\end{align*}
where $\pi_{\bar \r, \trans}$ is the optimal policy w.r.t.\ the posterior mean $\bar \r = \E_{\r \sim \bel}[\r]$ and the transition function $\trans$. If $\max_{\trans \in \Trans} \min_{\pi \in \Pi} \bayesregret_\bel (\trans, \pi) = 0$, it then follows that $\max_{\trans \in \Trans} \bayesregret_\bel (\trans, \pi_{\bar \r, \trans}) = 0$, i.e.\ $\val_{\r, \trans}^* = \val_{\r, \trans}^{\pi_{\bar \r, \trans}}$ for all $\r \in \supp(\bel)$ and $\trans \in \Trans$. This must imply that $\val_{\truereward, \trans}^* = \val_{\truereward, \trans}^{\pi_{\bar \r, \trans}}$ for all $\trans \in \Trans$. In other words, $\bar \r$ is optimal for all $\trans \in \Trans$ (under the initial state distribution $\omega$). 
\end{proof}

\section{Algorithms}\label{appendix:algorithms}

\subsection{Finding Maximin Environments}

\paragraph{Structured Environments.}
In the structured environment setting, we can independently choose a transition distribution for each state-action pair. Extended Value Iteration for Structured Environments, provided in Algorithm~\ref{algorithm:extended_VI}, estimates the maximin environment by iteratively making state-individual transition choices. At each stage, and for each state in the state space, we determine the state-transition matrix that maximises the regret for this given state. At the end, we return the environment corresponding to the transition matrices selected in the final stage.

\paragraph{Arbitrary Environments.}
The algorithm, Environment Design with Arbitrary Environments, is provided in Algorithm~\ref{algorithm:brute_force_ed}. It determines an estimate of the maximin environment by computing another approximation of the Bayesian regret. This approximation is done by running full policy optimisation steps in order to recover the policy values for each environment-reward combination. To obtain the Bayesian regret of an environment, we take the sum of the reward losses for each point estimate rewards: the difference between the value of a policy trained with a point estimate reward function and the value of a policy trained with the best guess. We return the environment with the highest Bayesian regret.
To alleviate the computation burden brought by this method, it is possible to sample a subset of environments for selection, rather than the full set of possible environments.

\paragraph{$\edairl$ with Arbitrary Environments.} 
Note that in the special situation running $\edairl$, when estimating the minimax environment at the second expert-learner round (Algorithm~\ref{algorithm:ed-airl} line~\ref{algorithm:ed-airl:environment_selection} when $k=2$) we end up with two possibly different estimations of the same reward function : $R_{1:1}$ and $\r_1$. In this case, it makes sense to consider the absolute reward loss when running environment design (Algorithm~\ref{algorithm:brute_force_ed} line~\ref{algorithm:brute_force_ed:approx_reward_loss}) as it let us discard environments that exhibit no information at all, while only having a single point estimate. Additionally, point estimates learned by $\airl$ must all be normalised with their minimum and maximum when computing the reward losses. 

\subsection{MaxEnt IRL with Multiple Environments}

$\airlme$, given in Algorithm~\ref{algorithm:adv_maxent_irl}, is the extension of $\airl$ to multiple environments.
To learn a common reward to $k$ chosen environments, we maintain $k$ generating policies, where policy $i$ only interacts with environment $i$, and $k$ discriminators, where each discriminator $i$ has its own shaping parameters $\phi_i$ but has shared reward approximator parameters $\theta$. Expert data for all $k$ considered environments are supposedly already collected. After generating trajectories with each policy, we train for a few steps discriminators $1, \dots, k$ to classify expert and generated data from environments $1, \dots, k$. Each policy $i$ is then trained for one policy optimisation step with their respective reward function $\r_{i, \theta, {\phi_i}}$. We repeat this procedure for a fitting number of iterations and then extract the learned common reward function.
Remark that $\airl$ has to be run in "state-only" mode, meaning that we learn a reward function that only depends on the current state. Since action meanings change with the transition dynamics, it does not make sense to include an action dependence to a reward that has to be common to multiple transition functions \cite{fu2017learning}.

\newpage
\begin{algorithm}[t]
\caption{Extended Value Iteration for Structured Environments}
\label{algorithm:extended_VI}
\begin{algorithmic}[1]
\STATE \textbf{input} environments $\Trans = \{\Trans_s\}_{s \in \S}$, empirical distribution $\hat \P$, best guess $\bar \r$ 
\STATE \textbf{repeat} until $\val_{\bestR}$ and $\val_{\r}$ converge:
\FOR{$s\in\S$}
\STATE $\trans_s = \argmax\limits_{\trans_s \in \Trans_s} \Big\{  \E_{\r\sim \hat \P} \big[ \max\limits_{a \in \A} \trans_{s, a}^\top  \val_\r\big] - \max\limits_{b \in \A} \trans_{s, b}^\top \val_{\bar \r}\Big\}$
\STATE $\val_{\r} (s) = \max\limits_{a \in \A} \r(s) + \gamma \trans_{s, a}^\top \val_\r$ for every $\r \sim \hat \P$ 
\STATE $\val_{\bar \r} (s) = \max\limits_{b \in \A} \bar \r(s) + \gamma \trans_{s, b}^\top  \val_{\bar \r}$

\ENDFOR 
\RETURN environment $\trans = \{ \trans_s\}_{s \in \S}$ 
\end{algorithmic}
\end{algorithm} 

\begin{algorithm}[t]
\caption{Environment Design with Arbitrary Environments} \label{algorithm:brute_force_ed}
\begin{algorithmic}[1]
\STATE \textbf{input\,} set of environments $\Sim$, point estimates $\{ \r_1, \dots, \r_k\}$, best guess $\bar \r$ 
\STATE \textcolor{gray}{// if necessary, sample a subset $\Sim_{\subset}$ from $\Sim$}
\FOR{$\trans \in \Sim$}
\STATE calculate $\bar \pi \in \argmax_{\pi} \val_{\bestR, T}^\pi$  \hfill (policy optimisation)
\FOR{ $\r \in \{ \r_1, \dots, \r_k\}$}
\STATE evaluate $\val_{\r, \trans}^{\bar \pi }$ \hfill (policy evaluation) \hspace{.21cm}
\STATE calculate $\val_{\r, \trans}^* = \max_{\pi} \val_{\r, \trans}^\pi$ \hfill (policy optimisation) 
\STATE $\ell(\r) = \val_{\r, \trans}^* - \val_{\r, \trans}^{\smash{\bar \pi}}$ \label{algorithm:brute_force_ed:approx_reward_loss}
\ENDFOR
\STATE $\bayesregret(\trans) = \sum_{\r \in \{ \r_1, \dots, \r_k\}} \ell(\r)$
\ENDFOR
\STATE \textbf{return\,} $\trans^* = \argmax_{\trans\in \Trans} \bayesregret(\trans)$ 
\end{algorithmic}
\end{algorithm} 
\begin{algorithm}[ht]
\caption{$\airlme$ (\airl{} with Multiple Environments)}
\label{algorithm:adv_maxent_irl}
\begin{algorithmic}[1]
\STATE \textbf{input\,} Observations $\data = (\data_1, \dots, \data_k)$ with $\data_i = (\tau_i, \trans_i)$ 
\STATE Initialise policies $\pi_1, \dots, \pi_k$ and discriminators $D_{1, \theta, {\phi_1}}, \dots, D_{k, \theta, {\phi_k}}$
\FOR{$t = 0, 1, \dots$}
\STATE Collect trajectories $\tau^G_{i, j} = (s_0, a_0, \dots, s_H, a_H)$ by executing $\pi_i$ in $\trans_i$ for $i \in [k]$.
\STATE Train discriminators $D_{1,\theta, {\phi_1}} \dots, D_{k, \theta, {\phi_k}}$ to classify expert data $\tau_1, \dots, \tau_k$ from samples $\{\tau^G_{1, j}\}_j, \dots, \{\tau^G_{k, j}\}_j$, respectively, via logistic regression with shared parameter $\theta$. \label{algline:discriminator_update}
\STATE Update reward $\r_{i, \theta, {\phi_i}}(s, a, s') \leftarrow\log D_{i, \theta, {\phi_i}}(s, a, s') - \log (1-D_{i, \theta, {\phi_i}} (s, a, s'))$ for $i\in [k]$. 
\STATE Update $\pi_1, \dots, \pi_k$ with respect to $\r_{1, \theta, {\phi_1}},\dots,\r_{k, \theta, {\phi_k}}$ using any policy optimisation method. \label{algline:policy_update}
\ENDFOR
\RETURN $g_{\theta}$ \hfill (extract the reward approximator)
\end{algorithmic}
\end{algorithm}

\section{Additional Experimental Results \label{appendix:experiment_results}}

\subsection{Learning Robust Reward Functions on Randomly Generated MDPs}\label{subsection:exp_robustness}

The learner is provided with a set of demo environments they can select for a demonstration. Afterwards, the agent is evaluated on a set of test environments. 
The performance on the test set captures the generalisation ability of the learned rewards to new dynamics.    


\begin{figure}[t]
     \centering
         \begin{subfigure}[b]{0.45\textwidth}
         \centering
         \includegraphics[width=\textwidth]{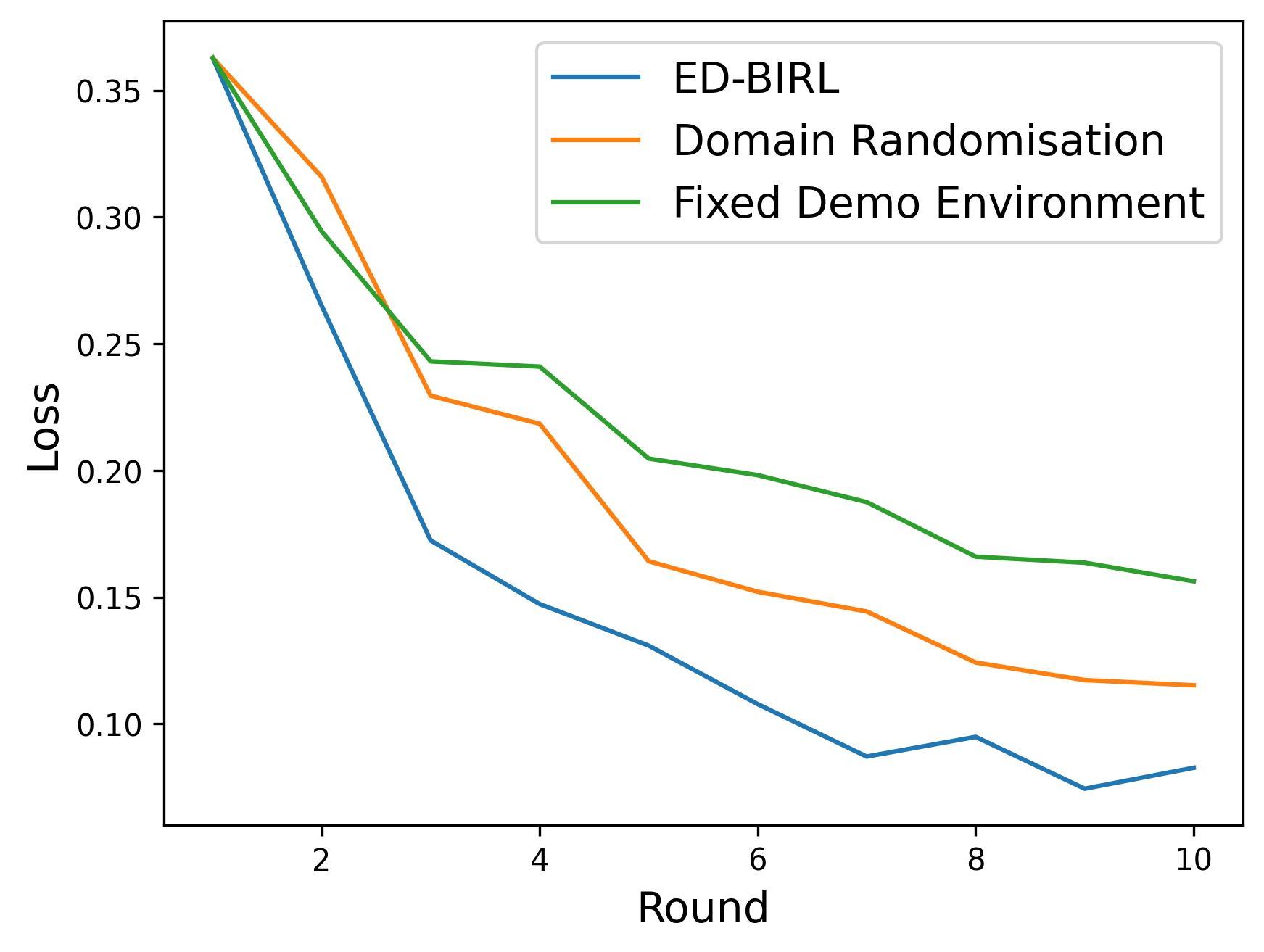}
         \caption{Average utility \emph{loss} of \edbirl, \texttt{Domain Randomisation}, and \texttt{Fixed Environment\,IRL} over $10$ rounds. The learned rewards are evaluated on a set of test environments that differ from the base environment by at most $\rhotest =0.5$.}
         \label{subfig:robustness_rounds}
     \end{subfigure}
      \hfill
    \begin{subfigure}[b]{0.45\textwidth}
         \includegraphics[width=\textwidth]{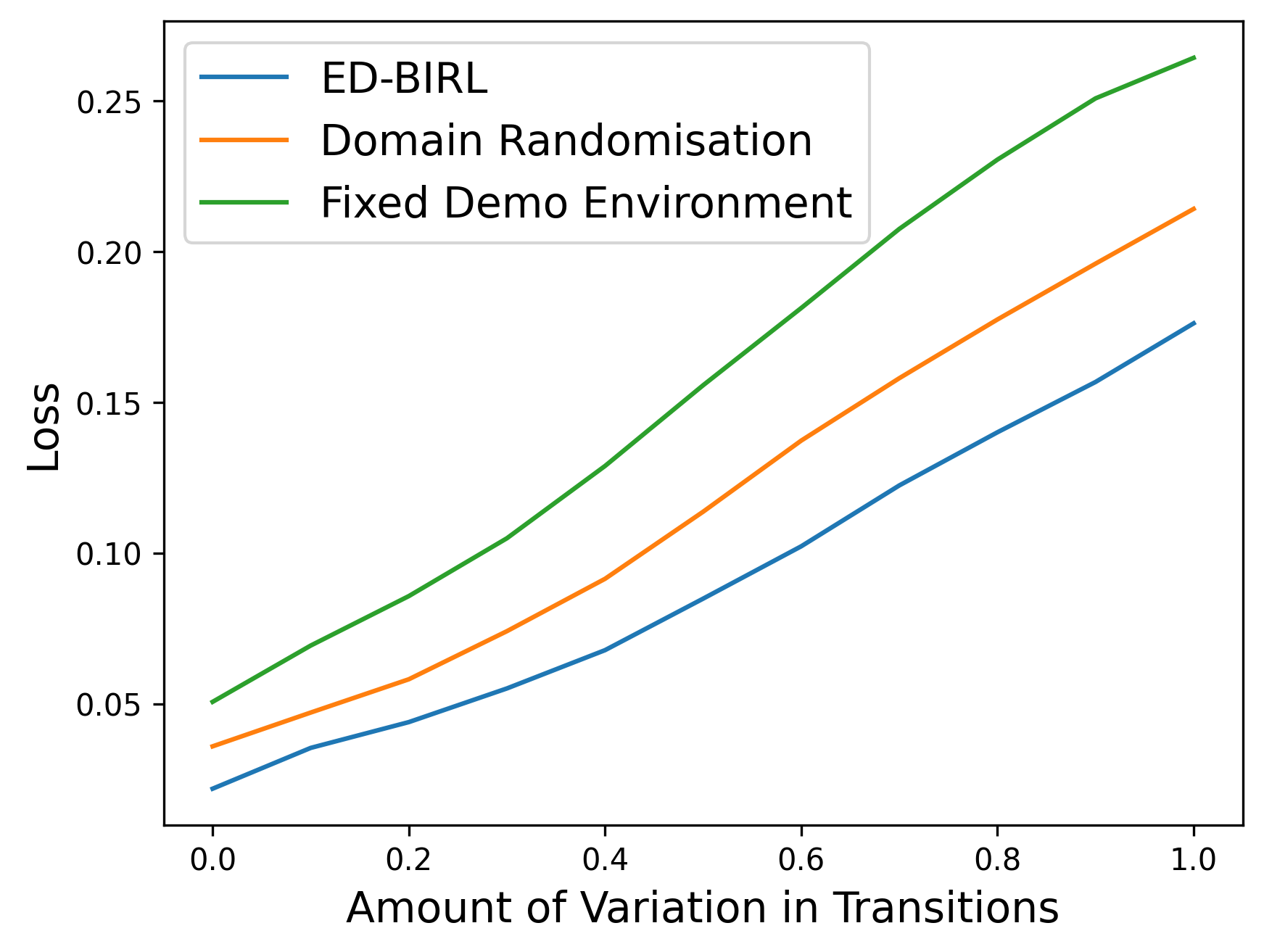}
        \caption{Along the $x$-axis we increase $\rhotest$, i.e.\ the amount of variation in the test environments. We evaluate the learned reward functions after $10$ rounds of interaction with the expert, i.e.\ the final reward estimate from Figure~\ref{subfig:robustness_rounds}.}
        \label{subfig:robustness_variation}
     \end{subfigure}
     \caption{On a randomly generated MDP task, we evaluate  the robustness of reward estimates learned by \edbirl, \texttt{Domain Randomisation}, and \texttt{Fixed Environment\,IRL}, respectively. }
     \label{fig:robustness}
\end{figure}

\paragraph{Experimental Setup.}
We first randomly generate a base MDP $(\S, \A, \basetrans, \truereward, \gamma, \omega)$ with base transition function $\basetrans$. We then construct the set of possible demo environments, here denoted $\Transdemo$ instead of $\Trans$ to clearly distinguish between demo and test environments, by sampling state-transition functions that differ from the base transitions $\basetrans$ by at most some value $\rhodemo$ in terms of $\ell_\infty$-distance. 
In our experiments, we set the maximum amount of variation in the demo environments to $\rhodemo = 0.5$. 
Similarly, we create a set of test environments $\Transtest$ with a maximum amount of perturbation $\rhotest$ on which we evaluate the learned reward functions. 
For all three approaches, we evaluate the posterior mean, which is computed using \texttt{BIRL}. 
For all $\trans\in \Transtest$, we optimise a policy w.r.t.\ the posterior mean and $\trans$ and evaluate the computed policy under the true reward function $\truereward$ and transition function~$\trans$. Finally, we average the results over all environments in $\Transtest$.   
We want to emphasise that the way we construct $\Transdemo$ and $\Transtest$, these sets are completely disjunct except for the base transition function, i.e.\ $\Transdemo \cap \Transtest = \{\basetrans\}$. We therefore \emph{do not} observe the expert in the environments that we evaluate our approaches on. 







\paragraph{Results.} In Figure~\ref{subfig:robustness_rounds}, we observe that \edbirl~outperforms domain randomisation and learning from a fixed environments over the course of all rounds. As expected, the loss of all three approaches increases the more diverse the test environments are and the more they differ from the base environment, which can be seen in Figure~\ref{subfig:robustness_variation}. Interestingly, even for $\rhotest=0$, i.e.\ evaluation on the base environment only, \edbirl~slightly outperforms learning directly from the fixed base environment suggesting a superior sample-efficiency of $\edbirl$.

\begin{figure}[t]

\centering
        
    \begin{subfigure}[b]{0.32\textwidth}
         \centering
         \begin{tikzpicture}
            \draw (0, 0) node[inner sep=0] {
            \includegraphics[width=0.45\textwidth]{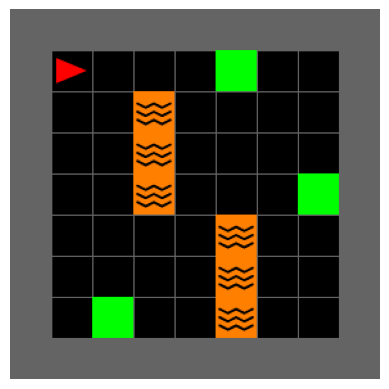}};
            \draw (0, 1.3) node {\tiny Chosen Maze};
            \draw (-1.3, 0) node {\tiny 1};
        \end{tikzpicture}
         \begin{tikzpicture}
            \draw (0, 0) node[inner sep=0] {
            \includegraphics[width=0.45\textwidth]{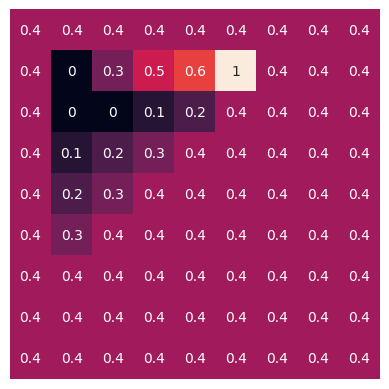}};
            \draw (0, 1.3) node {\tiny Estimated Rewards};
        \end{tikzpicture} \\
         \begin{tikzpicture}
            \draw (0, 0) node[inner sep=0] {\includegraphics[height=0.45\textwidth]{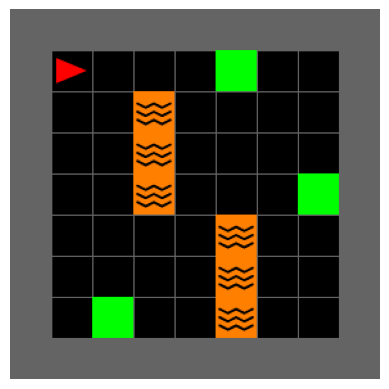}};
            \draw (-1.3, 0) node {\tiny 2};
         \end{tikzpicture}
         \includegraphics[height=0.45\textwidth]{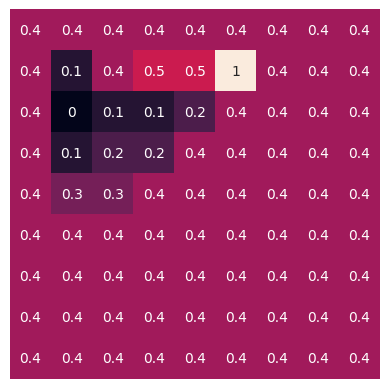}\\
         \begin{tikzpicture}
            \draw (0, 0) node[inner sep=0] {\includegraphics[height=0.45\textwidth]{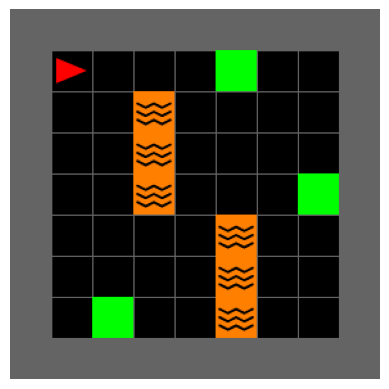}};
            \draw (-1.3, 0) node {\tiny 3};
         \end{tikzpicture}
         \includegraphics[height=0.45\textwidth]{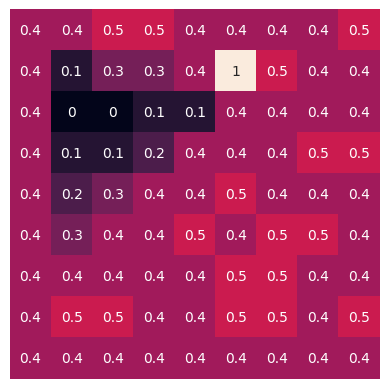}
         \caption{\texttt{Fixed Environment\,IRL}}
                           \label{subfig:lava_bars_fixed}

     \end{subfigure}
           \hfill
    \begin{subfigure}[b]{0.32\textwidth}
         \centering
         \begin{tikzpicture}
            \draw (0, 0) node[inner sep=0] {
            \includegraphics[width=0.45\textwidth]{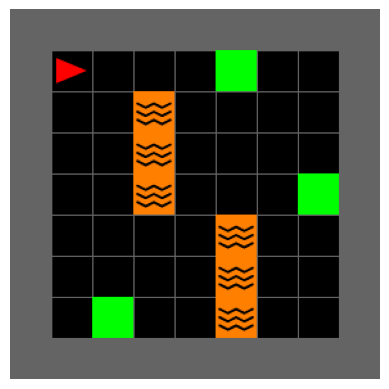}};
            \draw (0, 1.3) node {\tiny Chosen Maze};
        \end{tikzpicture}
         \begin{tikzpicture}
            \draw (0, 0) node[inner sep=0] {
            \includegraphics[width=0.45\textwidth]{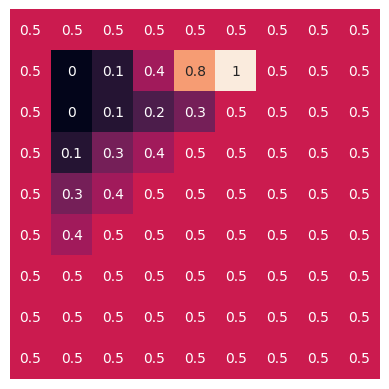}};
            \draw (0, 1.3) node {\tiny Estimated Rewards};
        \end{tikzpicture} \\
         \includegraphics[height=0.45\textwidth]{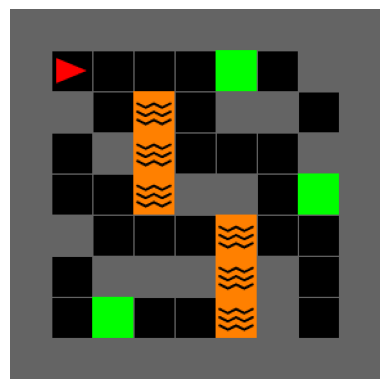} 
         \includegraphics[height=0.45\textwidth]{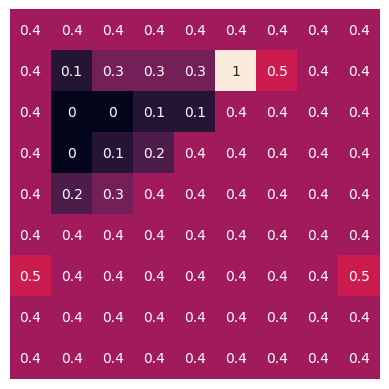}\\
         \includegraphics[height=0.45\textwidth]{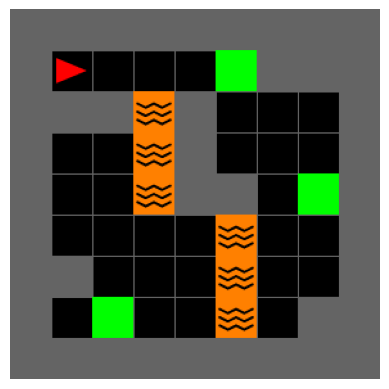} 
         \includegraphics[height=0.45\textwidth]{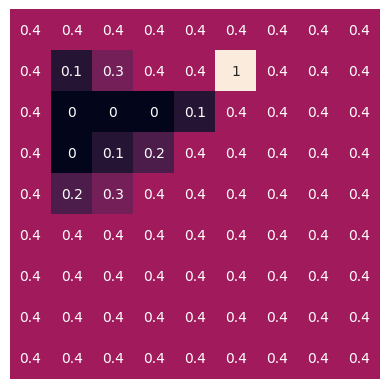}
         \caption{\texttt{Domain Randomisation}}
                           \label{subfig:lava_bars_dom_ran}
     \end{subfigure}
     \hfill
     \begin{subfigure}[b]{0.32\textwidth}
         \centering
         \begin{tikzpicture}
            \draw (0, 0) node[inner sep=0] {\includegraphics[width=0.45\textwidth]{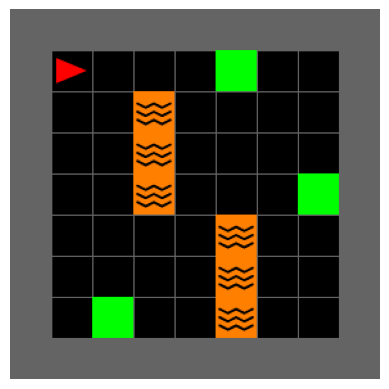}};
            \draw (0, 1.3) node {\tiny Chosen Maze};
        \end{tikzpicture}
         \begin{tikzpicture}
            \draw (0, 0) node[inner sep=0] {
            \includegraphics[width=0.45\textwidth]{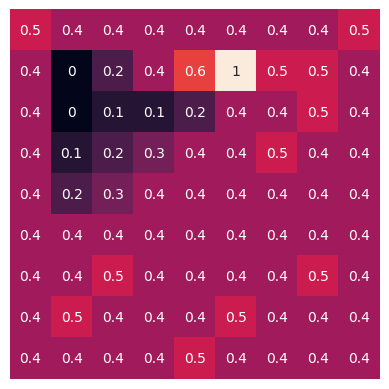}};
            \draw (0, 1.3) node {\tiny Estimated Rewards};
        \end{tikzpicture} \\
         \begin{tikzpicture}
            \draw (0, 0) node[inner sep=0] {\includegraphics[width=0.45\textwidth]{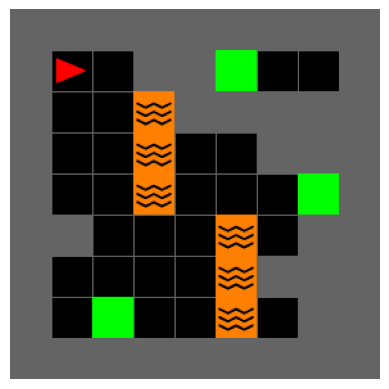}};
        \end{tikzpicture}
         \includegraphics[height=0.45\textwidth]{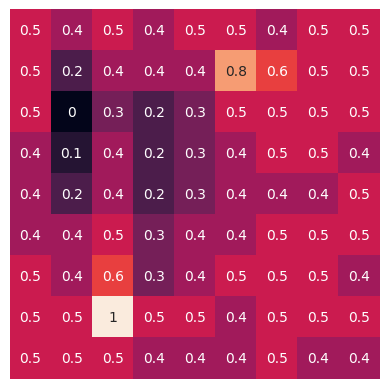}\\
         \begin{tikzpicture}
            \draw (0, 0) node[inner sep=0] {\includegraphics[width=0.45\textwidth]{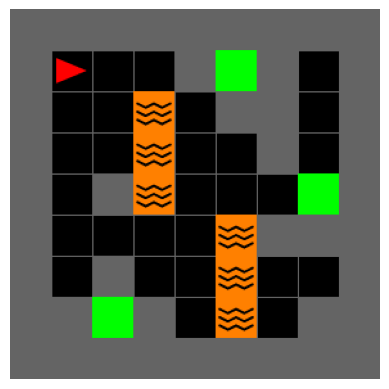}};
        \end{tikzpicture}
         \includegraphics[height=0.45\textwidth]{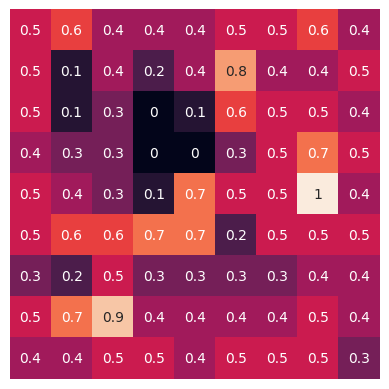}
         \caption{\texttt{ED-BIRL}}
                  \label{subfig:lava_bars_ed_birl}

     \end{subfigure}
     \caption{Comparison of \texttt{ED-BIRL}, \texttt{Fixed Environment\,IRL}, and \texttt{Domain Randomisation} on the first three rounds of another version of the maze problem seen in Section~\ref{subsection:exp_maze}.} 
     \label{fig:exp_maze_d-f}
\end{figure}
\begin{figure}[t]
     \centering
         \begin{subfigure}[b]{0.40\textwidth}
         \centering
         \includegraphics[width=0.99\textwidth]{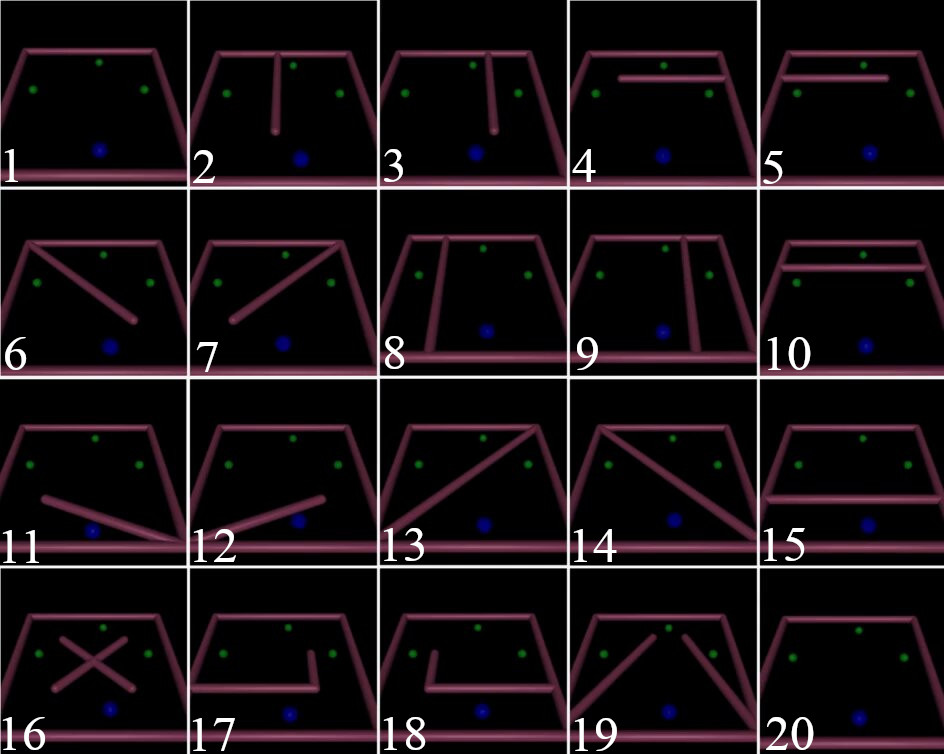} 
         \vspace{26px}
         \caption{Continuous maze demo environments.
         \label{fig:continuous_maze_demo}}
         \end{subfigure}
         \hfill
         \begin{subfigure}[b]{0.56\textwidth}
         \centering
         \includegraphics[width=0.99\textwidth]{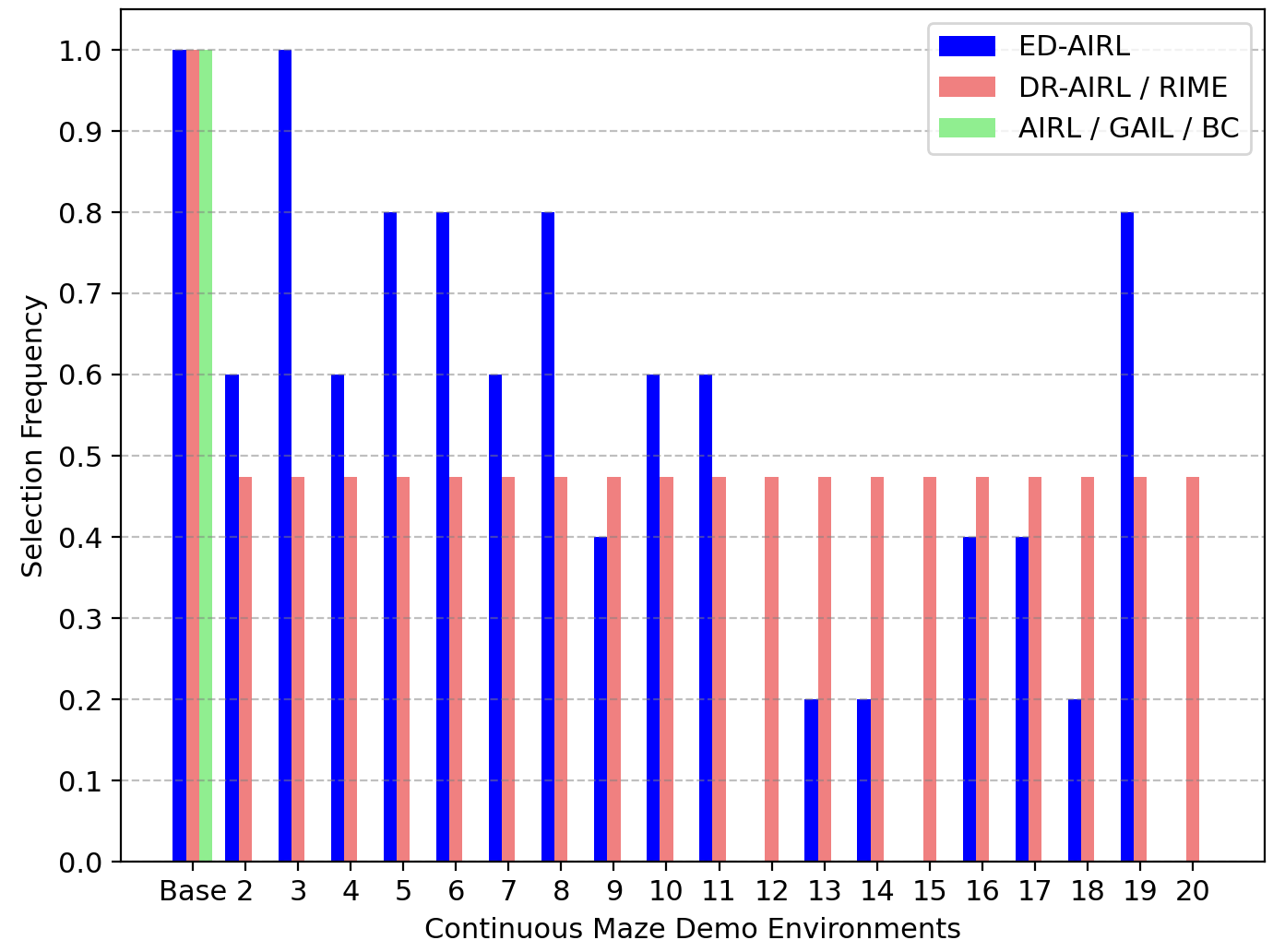} 
        \caption{Frequency at which each continuous maze environment was picked for expert demonstrations.\label{fig:continuous_mazes_selection}}
         \end{subfigure}
        \caption{For the continuous maze task, the learner can observe experts in $n=10$ different environments. $\airl$, $\gail$ and $\bc$ only observe the expert on the base environment, while $\edairl$ and $\drairl$ can observe the expert on any environments in the demo set. For a fair comparison, we force multi-environment approaches to always include the base environment in their selection. Environment labeled 20 is designed to be uninformative: the agent is disabled and can not move around, an expert policy is thus indistinguishable from any other policy under this setting. \label{fig:maze_experiment}}
\end{figure}
\begin{figure}[t]
     \centering
        \includegraphics[width=0.4\linewidth]{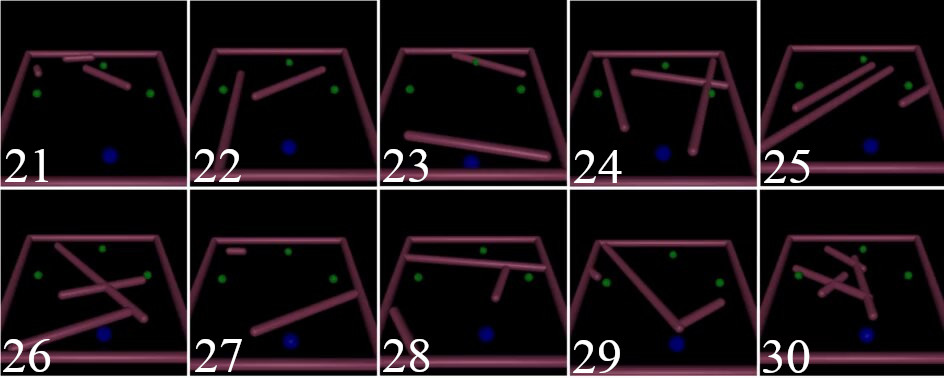} 
        \caption{Continuous maze test environments. \label{fig:continuous_mazes_test}}
        \vspace{-0.2cm}
\end{figure}



\subsection{Recovering the True Reward Function}

We provide an additional visualisation of the learned rewards on another maze problem considered in the experiments from Section~\ref{subsection:exp_maze} in Figure~\ref{fig:exp_maze_d-f}. Results on this version of the maze comforts our past observations. $\edbirl$, Figure~\ref{subfig:lava_bars_ed_birl}, recovers the location of all three goal states and some lava states, after three round. We see that each round, the environment generated by $\edbirl$ forces the expert to visit a goal state which is yet unknown to the reward function. In contrast, \texttt{Fixed Environment\,IRL} and \texttt{Domain Randomisation}, Figures ~\ref{subfig:lava_bars_fixed},\ref{subfig:lava_bars_dom_ran}, only retrieve the location of one goal state. Even when increasing the number of rounds (see Figure~\ref{fig:dr_additional_experiment_seed_0}), \texttt{Domain Randomisation} does not manage to recover all three goals, pointing out clearly the lower sample efficiency of the method.
For the initial maze and the randomly generated mazes, the (optimal) expert trajectory always corresponds to reaching the top right goal state. For this reason, those random environment designs never lead to discovering additional performance-relevant aspects of the reward function.

\subsection{Learning Robust Reward Functions for Continuous Control}

With Figure~\ref{fig:maze_experiment}, we provide a more interpretable visualisation of how $\edairl$ selects environments. After observing the expert on the base environment, $\drairl$ selects environments at uniform, while $\edairl$ performs a curated selection. Figure~\ref{fig:continuous_mazes_selection} shows how often each environment is selected by each approach, on the continuous maze task. We observe that on the 5 times we ran $\edairl$, it consistently avoided environments similar to the base environment (12) and uninformative environments (15 and 20). This is a direct improvement from $\drairl$. We also provide a visualisation of the rewards learned by $\edairl$ and $\drairl$ on the continuous maze.

We provide more detailed results in Table~\ref{tab:detailed_tables_continuous}. Experiments on the MuJoCo Ant environment are also included. Based on the quantiles on both continuous maze and Swimmer, $\edairl$ reliably finds a reward function that solves at least 25\% of the mazes on both sets, which is not ensured by $\airl$. More importantly, this manifests that it is in general safer to learn from a reward function estimated by $\edairl$ as its 25\%-quantile is consistently higher.

\section{Additional Experimental Details \label{appendix:experiment_details}}


\paragraph{Recovering the True Reward Function.} For the experiments in Section~\ref{subsection:exp_maze}, we let the learner observe two trajectories for each maze. This was done in order to speed up the inference of \texttt{BIRL} and reduce the computational cost. The expert was modeled by a Boltzmann-rational policy and thus uniformly selected an optimal action when there were several optimal ones in a given state. 


\paragraph{Learning Robust Reward Functions for Continuous Control.}
For the experiments in Section~\ref{subsection:exp_continuous}, We allow each approach to learn with unlimited interaction with their environments, but evaluate their learned rewards by optimising policies for a fixed amount of environment steps. This is justified by the fact that we want to compare the best reward estimation offered by each approach, only constrained by the amount of expert data $m$. Besides, because increasing the amount of observable environment on $\edairl$ quadratically rises the amount of policy optimisation required, we limited the selection size to $n=10$ for the continuous maze task and $n=5$ for the remaining tasks. Each time we select a new environment in $\edairl$ with Algorithm~\ref{algorithm:brute_force_ed}, we sample a subset $\Trans_\subset$ of size 10 and 5, for the continuous maze task and other tasks respectively. For all $\airl$-based algorithms, we used a two-layer ReLU network with 32 units for the state-only reward approximator and shaping functions. We also smoothed the estimated rewards outputted by $\airl$ and $\airlme$ by taking the average of the estimated reward functions from the last 10 discriminator-generator rounds.

We chose to limit ourselves to the default MuJoCo environment settings as much as possible. This meant excluding robot positions from observations, except in Swimmer where it was deemed necessary.
We utilised the implementations from \cite{gleave2022imitation} for $\gail$ and $\bc$, initially tuned with the inclusion of the position in the observation, explaining some variations in our results.
Regarding $\rime$, we attempted to tune the original implementation from \cite{chae2022robust} and used their weight-shared discriminator approach. 
Lastly, $\airl$ is known to not perform on Ant. \cite{fu2017learning} and \cite{gleave2022imitation} both considered modified versions of the environments, we in contrast kept the default settings in order to observe whether multi-environment approaches could improve the results or not.

All of our policies were optimised with Proximal Policy Optimisation \cite{schulman2017proximal}. Within the scope of each task, demo and test environment experts were trained with identical hyperparameters and for an equal amount of timesteps. For better hindsight, average values of random and expert policies on base environments are given in Table~\ref{tab:random_expert_levels}. 

\paragraph{Continuous maze.}
The task consists in reaching all accessible green targets in whichever order. The agent observes its current position and whether each goal was already reached or not. While rewarded with 1 point for collecting a target, the agent is incurred with an action cost to motivate optimal path making. The environment is considered solved when all reachable targets were collected before the episode of fixed length ends. We provide a visual of the continuous maze demo and test set in Figure~\ref{fig:continuous_maze_demo} and Figure~\ref{fig:continuous_mazes_test}.

\paragraph{Learning Robust Reward Functions on Randomly Generated MDPs.}
We randomly generated an MDP with 50 states and 4 actions using a Dirichlet distribution for the transitions and a Beta distribution for the reward function. For each state we let the demo set of environments contain 15 choices. The size of the test environments was set to be $|\Transtest| = 500$. Every round, the learner got to select a demo environment and observe a single expert trajectory in that environment. We limited the amount of deviation from the base transitions in our experiments according to $\rhodemo$ and $\rhotest$. In particular, note that any choice of $\rhodemo$ implies that $\lVert \basetrans - \trans\rVert_{\infty} = \max_{s, a} \lVert \basetrans (\cdot \mid s, a) - \trans(\cdot \mid s, a)\rVert_1 \leq \rhodemo$ for all $\trans \in \Transdemo$. 
The results were averaged over $5$ complete runs, i.e.\ for $5$ randomly generated problem instances. 

\paragraph{Implementation.}
The code used for all of our experiments is available at \href{https://github.com/Ojig/Environment-Design-for-IRL}{github.com/Ojig/Environment-Design-for-IRL}.

\paragraph{Compute.}
Three AMD EPYC 7302P machines were used. Most of the computation time was spent running ED-AIRL, which requires a large amount of full policy optimisation steps. With 5 expert-learner rounds and depending on the MuJoCo environment ED-AIRL requires from half to a full day of wall-clock computation time on one machine.

\begin{table}[t]

\setlength{\tabcolsep}{2pt} 
\renewcommand{\arraystretch}{1.4} 
\centering
\resizebox{\textwidth}{!}{
\begin{tabular}{L{0.16\textwidth}|C{0.16\textwidth}|C{0.14\textwidth}|C{0.16\textwidth}|C{0.16\textwidth}|C{0.16\textwidth}}
\multicolumn{1}{L{0.16\textwidth}|}{} & 
\multicolumn{1}{C{0.16\textwidth}|}{Continous Maze} &
\multicolumn{1}{C{0.16\textwidth}|}{Swimmer} &
\multicolumn{1}{C{0.16\textwidth}|}{HalfCheetah} &
\multicolumn{1}{C{0.16\textwidth}|}{Hopper} &
\multicolumn{1}{C{0.16\textwidth}}{Ant} \\ \midrule
Random               & 1.14±0.80 & 9±12 & -246±76 & 16±16 & -34±23 \\ \cdashline{1-6}
Expert               & 2.83±0.01 & 174±9 & 3728±447 & 2723±275 & 2331±725
\end{tabular}
}
\caption{Average scores and their standard deviation, for random and expert policies on the base environments. Experts for base, demo and test environments for a given task were trained with the same hyperparameters for the same number of iterations.\label{tab:random_expert_levels}}
\end{table}
\begin{figure}[t]

\centering
        
    \begin{subfigure}[b]{0.32\textwidth}
         \centering
         \begin{tikzpicture}
            \draw (0, 0) node[inner sep=0] {
            \includegraphics[width=0.45\textwidth]{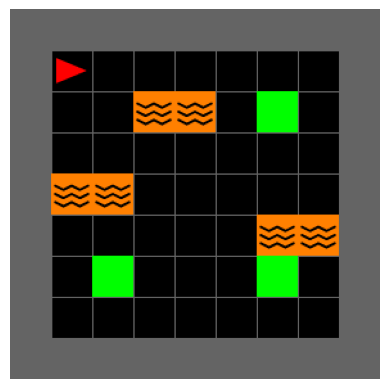}};
            \draw (0, 1.3) node {\tiny Chosen Mazes (ED)};
            \draw (-1.3, 0) node {\tiny 1};
        \end{tikzpicture}
         \begin{tikzpicture}
            \draw (0, 0) node[inner sep=0] {
            \includegraphics[width=0.45\textwidth]{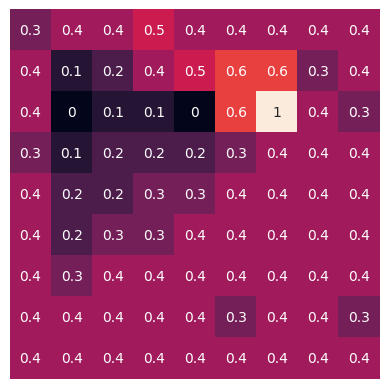}};
            \draw (0, 1.3) node {\tiny Estimated Rewards};
        \end{tikzpicture} \\
         \begin{tikzpicture}
            \draw (0, 0) node[inner sep=0] {\includegraphics[height=0.45\textwidth]{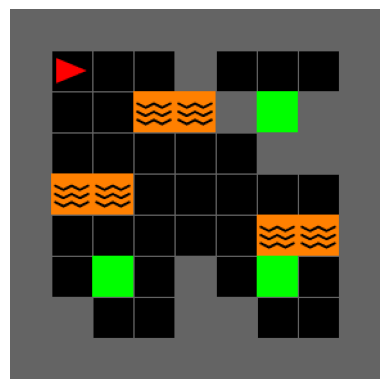}};
            \draw (-1.3, 0) node {\tiny 2};
         \end{tikzpicture}
         \includegraphics[height=0.45\textwidth]{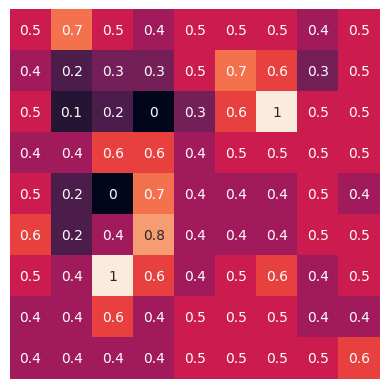}\\
         \begin{tikzpicture}
            \draw (0, 0) node[inner sep=0] {\includegraphics[height=0.45\textwidth]{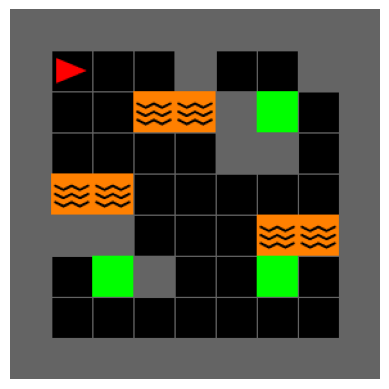}};
            \draw (-1.3, 0) node {\tiny 3};
         \end{tikzpicture}
         \includegraphics[height=0.45\textwidth]{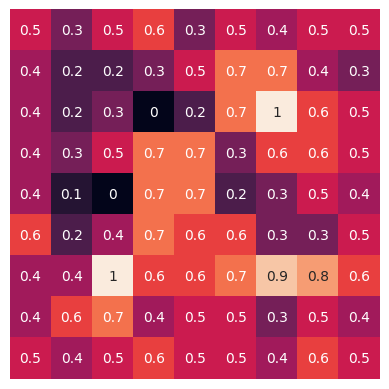}\\
         \begin{tikzpicture}
            \draw (0, 0) node[inner sep=0] {\includegraphics[height=0.45\textwidth]{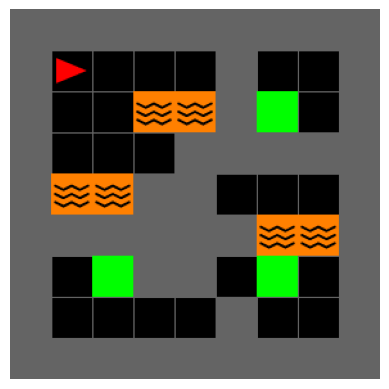}};
            \draw (-1.3, 0) node {\tiny 4};
         \end{tikzpicture}
         \includegraphics[height=0.45\textwidth]{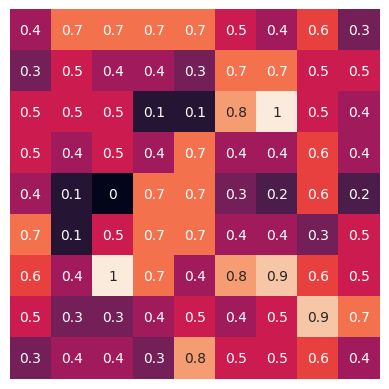}
         \caption{Run 1.}
                           \label{subfig:4_steps_ed_run1}
     \end{subfigure}
           \hfill
        \begin{subfigure}[b]{0.32\textwidth}
         \centering
         \begin{tikzpicture}
            \draw (0, 0) node[inner sep=0] {
            \includegraphics[width=0.45\textwidth]{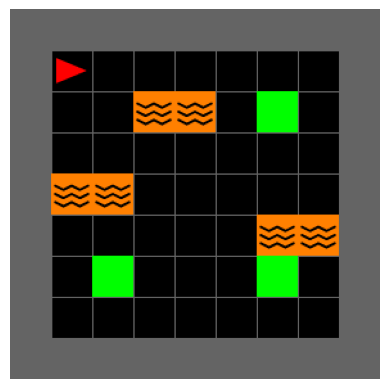}};
            \draw (0, 1.3) node {\tiny Chosen Mazes (ED)};
            \draw (-1.3, 0) node {\tiny 1};
        \end{tikzpicture}
         \begin{tikzpicture}
            \draw (0, 0) node[inner sep=0] {
            \includegraphics[width=0.45\textwidth]{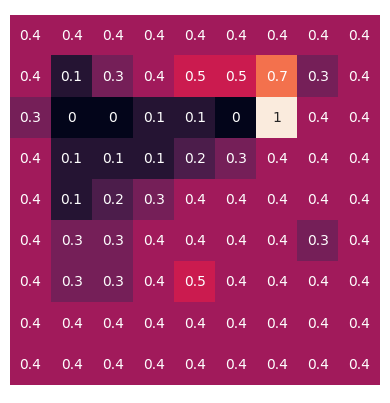}};
            \draw (0, 1.3) node {\tiny Estimated Rewards};
        \end{tikzpicture} \\
         \begin{tikzpicture}
            \draw (0, 0) node[inner sep=0] {\includegraphics[height=0.45\textwidth]{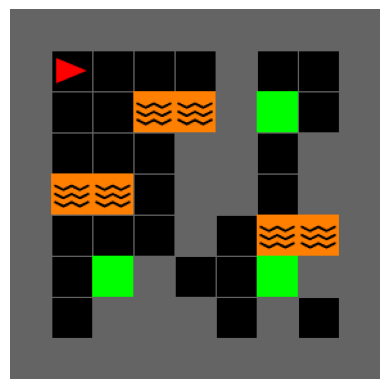}};
            \draw (-1.3, 0) node {\tiny 2};
         \end{tikzpicture}
         \includegraphics[height=0.45\textwidth]{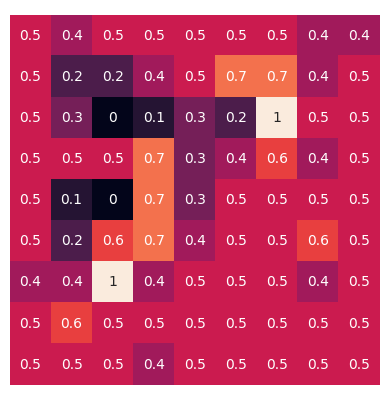}\\
         \begin{tikzpicture}
            \draw (0, 0) node[inner sep=0] {\includegraphics[height=0.45\textwidth]{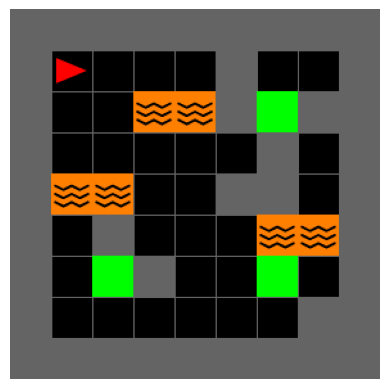}};
            \draw (-1.3, 0) node {\tiny 3};
         \end{tikzpicture}
         \includegraphics[height=0.45\textwidth]{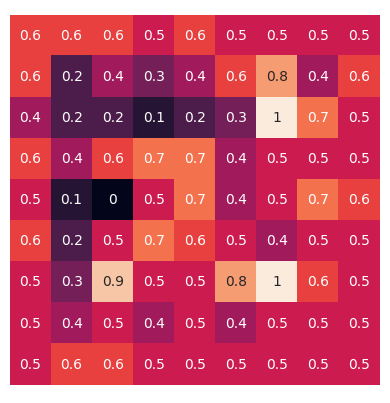}\\
         \begin{tikzpicture}
            \draw (0, 0) node[inner sep=0] {\includegraphics[height=0.45\textwidth]{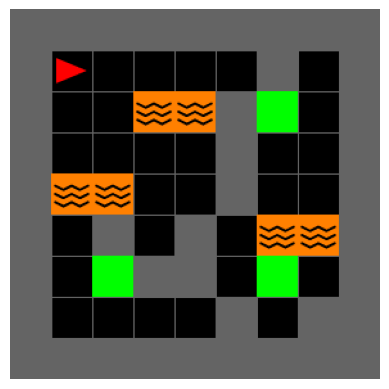}};
            \draw (-1.3, 0) node {\tiny 4};
         \end{tikzpicture}
         \includegraphics[height=0.45\textwidth]{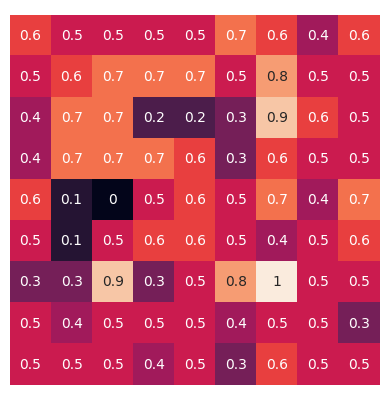}
         \caption{Run 2.}
                           \label{subfig:4_steps_ed_run2}
     \end{subfigure}
           \hfill
         \begin{subfigure}[b]{0.32\textwidth}
         \centering
         \begin{tikzpicture}
            \draw (0, 0) node[inner sep=0] {
            \includegraphics[width=0.45\textwidth]{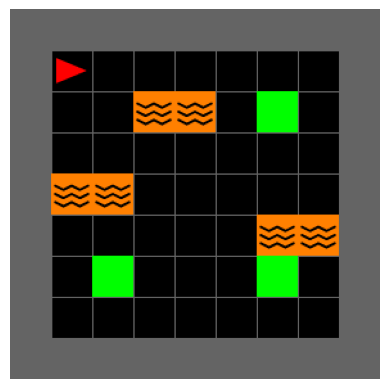}};
            \draw (0, 1.3) node {\tiny Chosen Mazes (ED)};
            \draw (-1.3, 0) node {\tiny 1};
        \end{tikzpicture}
         \begin{tikzpicture}
            \draw (0, 0) node[inner sep=0] {
            \includegraphics[width=0.45\textwidth]{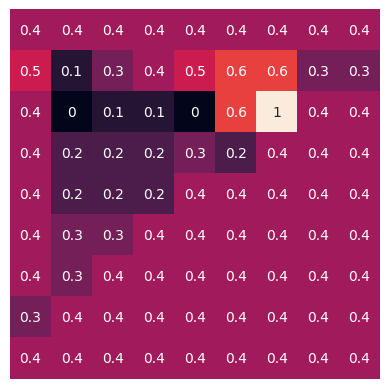}};
            \draw (0, 1.3) node {\tiny Estimated Rewards};
        \end{tikzpicture} \\
         \begin{tikzpicture}
            \draw (0, 0) node[inner sep=0] {\includegraphics[height=0.45\textwidth]{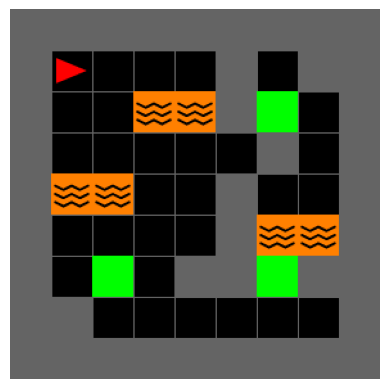}};
            \draw (-1.3, 0) node {\tiny 2};
         \end{tikzpicture}
         \includegraphics[height=0.45\textwidth]{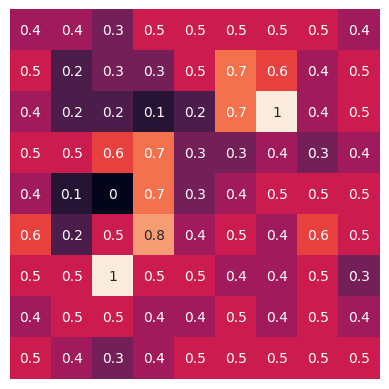}\\
         \begin{tikzpicture}
            \draw (0, 0) node[inner sep=0] {\includegraphics[height=0.45\textwidth]{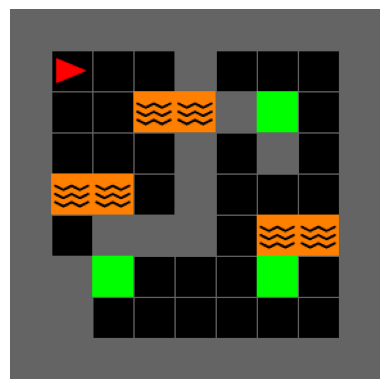}};
            \draw (-1.3, 0) node {\tiny 3};
         \end{tikzpicture}
         \includegraphics[height=0.45\textwidth]{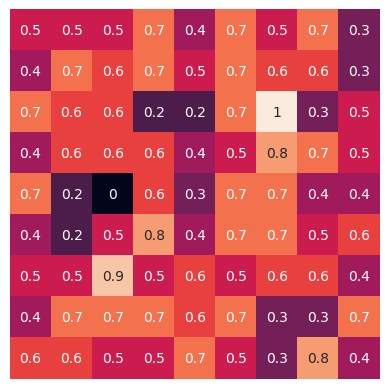}\\
         \begin{tikzpicture}
            \draw (0, 0) node[inner sep=0] {\includegraphics[height=0.45\textwidth]{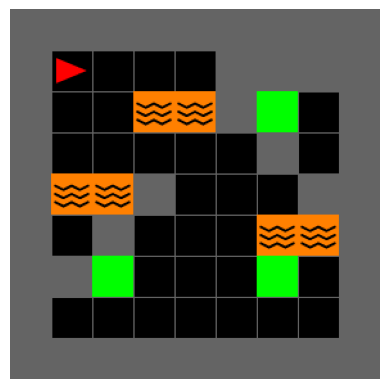}};
            \draw (-1.3, 0) node {\tiny 4};
         \end{tikzpicture}
         \includegraphics[height=0.45\textwidth]{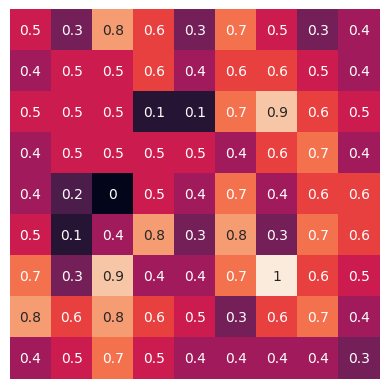}
         \caption{Run 3.}
                           \label{subfig:4_steps_ed_run3}
     \end{subfigure}
     \caption{Selected mazes and reward estimates of \texttt{ED-AIRL} over the course of 4 environment selections (two trajectories per environment), for 3 separate runs.} 
     \label{fig:exp_maze_d-f}
     \vspace{-0.5cm}
\end{figure}

\begin{figure}[t]
\begin{minipage}{0.9\columnwidth}
    
\centering
        
         \centering
         \begin{tikzpicture}
            \draw (0, 0) node[inner sep=0] {
            \includegraphics[width=0.14\textwidth]{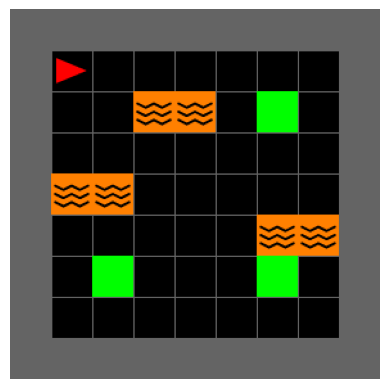}};
            \draw (0, 1.3) node {\tiny Sampled mazes (DR)};
            \draw (-1.3, 0) node {\tiny 1};
        \end{tikzpicture}
        \begin{tikzpicture}
            \draw (0, 0) node[inner sep=0] {
            \includegraphics[width=0.14\textwidth]{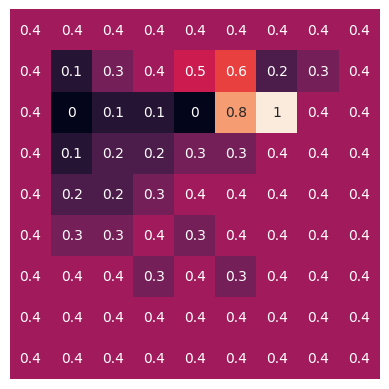}};
            \draw (0, 1.3) node {\tiny DR-BIRL};
        \end{tikzpicture}
        \begin{tikzpicture}
            \draw (0, 0) node[inner sep=0] {
            \includegraphics[width=0.14\textwidth]{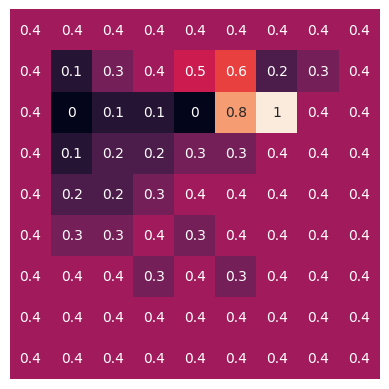}};
            \draw (0, 1.3) node {\tiny Coordinate-wise avg};
        \end{tikzpicture}
        \begin{tikzpicture}
            \draw (0, 0) node[inner sep=0] {
            \includegraphics[width=0.14\textwidth]{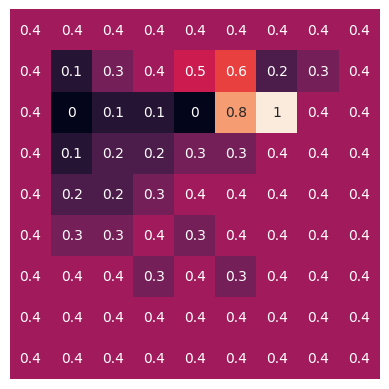}};
            \draw (0, 1.3) node {\tiny Coordinate-wise max};
        \end{tikzpicture}
        \\
         \begin{tikzpicture}
            \draw (0, 0) node[inner sep=0] {\includegraphics[width=0.14\textwidth]{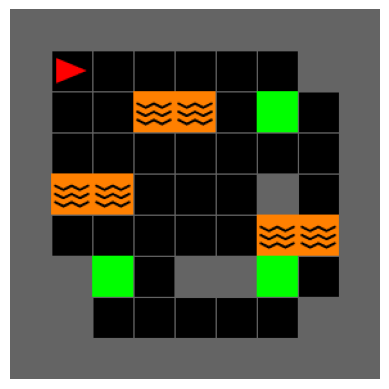}};
            \draw (-1.3, 0) node {\tiny 2};
         \end{tikzpicture}
          \includegraphics[height=0.14\textwidth]{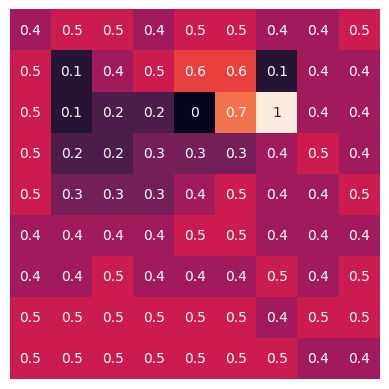} 
           \includegraphics[height=0.14\textwidth]{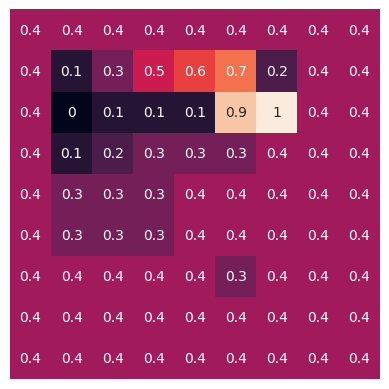} 
            \includegraphics[height=0.14\textwidth]{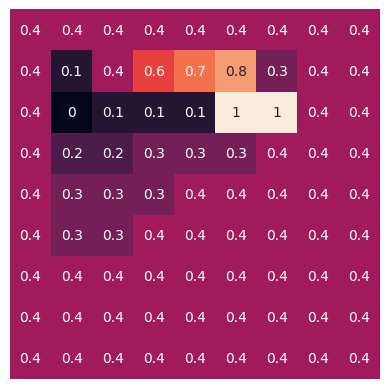} 
        \\
         \begin{tikzpicture}
            \draw (0, 0) node[inner sep=0] {\includegraphics[width=0.14\textwidth]{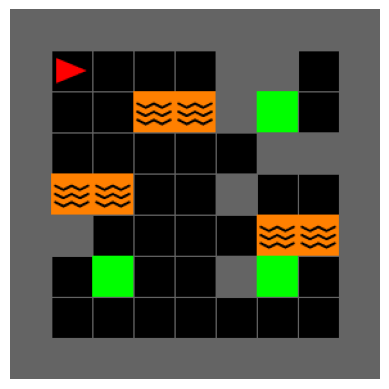}};
            \draw (-1.3, 0) node {\tiny 3};
         \end{tikzpicture}
         \includegraphics[height=0.14\textwidth]{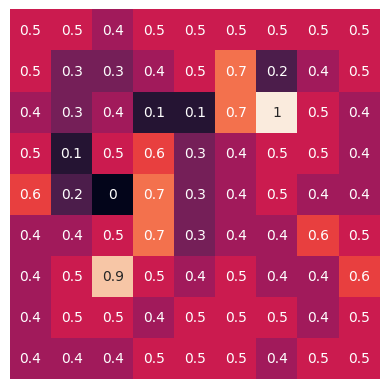} 
           \includegraphics[height=0.14\textwidth]{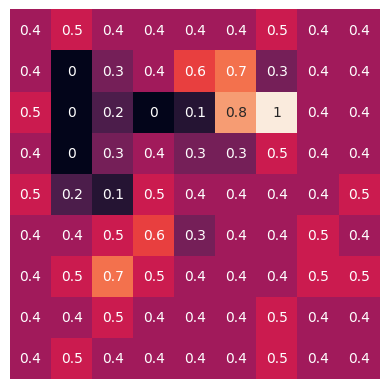} 
            \includegraphics[height=0.14\textwidth]{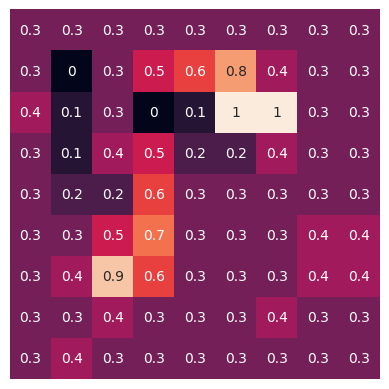}
        \\
         \begin{tikzpicture}
            \draw (0, 0) node[inner sep=0] {\includegraphics[width=0.14\textwidth]{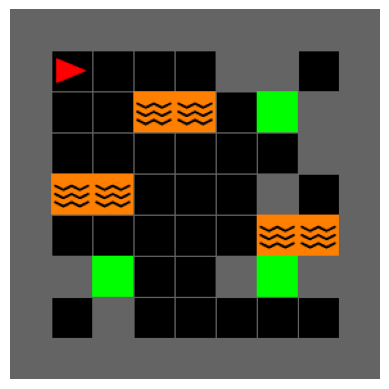}};
            \draw (-1.3, 0) node {\tiny 4};
         \end{tikzpicture}
         \includegraphics[height=0.14\textwidth]{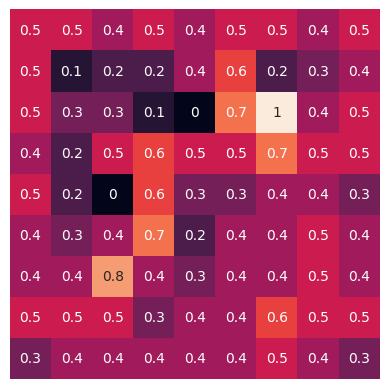} 
           \includegraphics[height=0.14\textwidth]{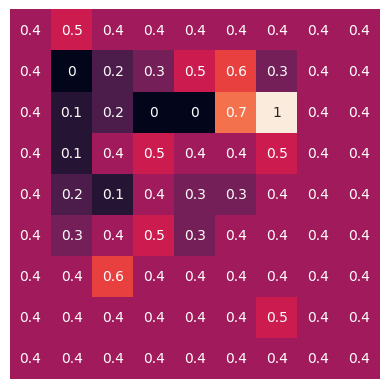} 
            \includegraphics[height=0.14\textwidth]{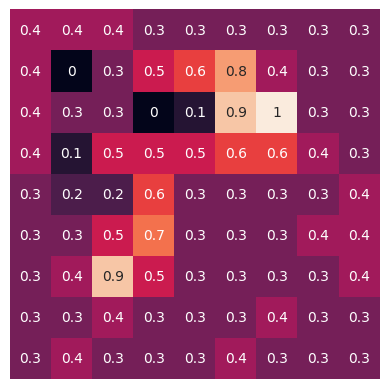} 
        \\
         \begin{tikzpicture}
            \draw (0, 0) node[inner sep=0] {\includegraphics[width=0.14\textwidth]{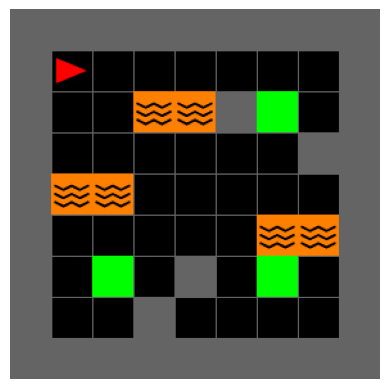}};
            \draw (-1.3, 0) node {\tiny 5};
         \end{tikzpicture}
         \includegraphics[height=0.14\textwidth]{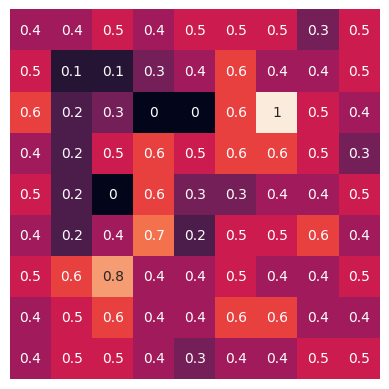} 
           \includegraphics[height=0.14\textwidth]{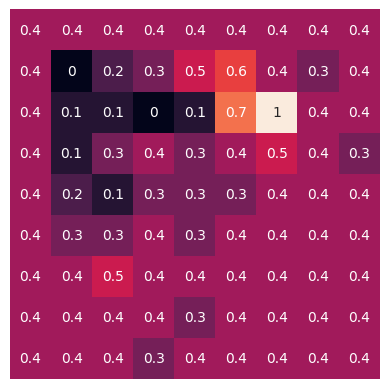} 
            \includegraphics[height=0.14\textwidth]{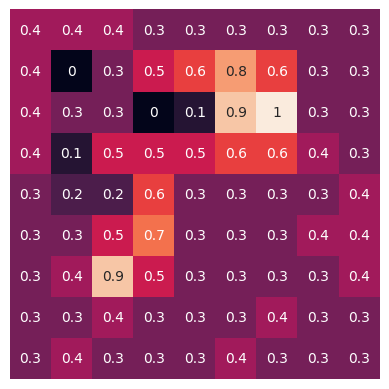}
        \\
         \begin{tikzpicture}
            \draw (0, 0) node[inner sep=0] {\includegraphics[width=0.14\textwidth]{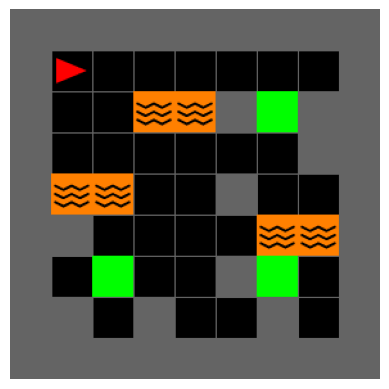}};
            \draw (-1.3, 0) node {\tiny 6};
         \end{tikzpicture}
         \includegraphics[height=0.14\textwidth]{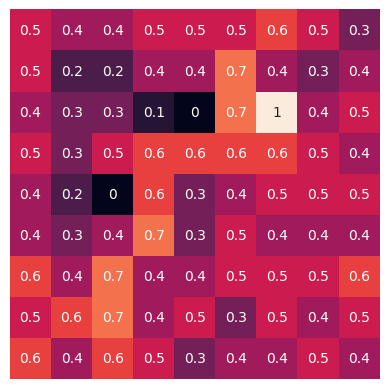} 
           \includegraphics[height=0.14\textwidth]{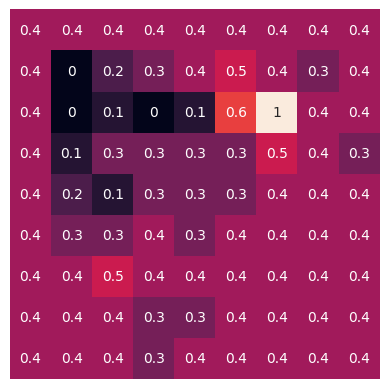} 
            \includegraphics[height=0.14\textwidth]{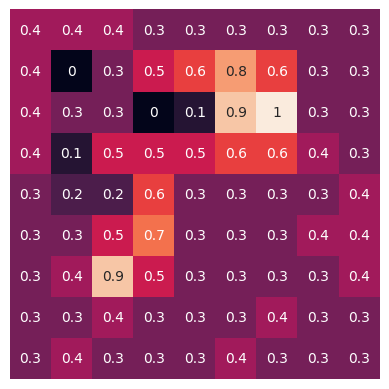} 
        \\
         \begin{tikzpicture}
            \draw (0, 0) node[inner sep=0] {\includegraphics[width=0.14\textwidth]{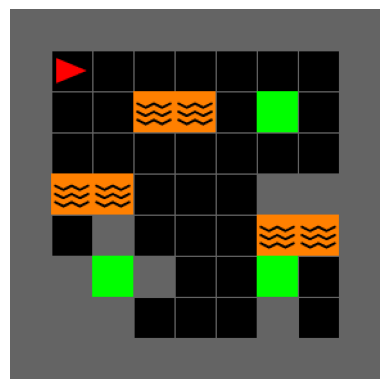}};
            \draw (-1.3, 0) node {\tiny 7};
         \end{tikzpicture}
         \includegraphics[height=0.14\textwidth]{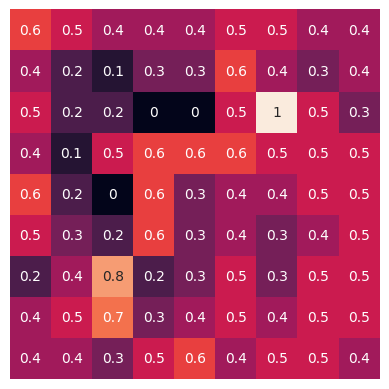} 
           \includegraphics[height=0.14\textwidth]{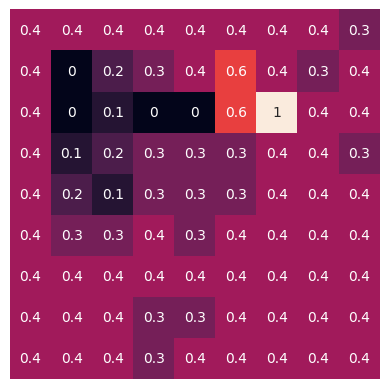} 
            \includegraphics[height=0.14\textwidth]{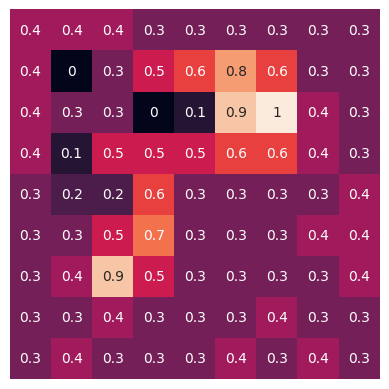} 
        \\
         \begin{tikzpicture}
            \draw (0, 0) node[inner sep=0] {\includegraphics[width=0.14\textwidth]{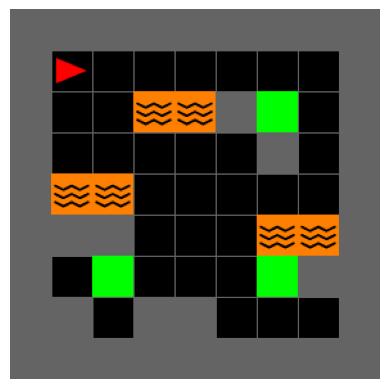}};
            \draw (-1.3, 0) node {\tiny 8};
         \end{tikzpicture}
         \includegraphics[height=0.14\textwidth]{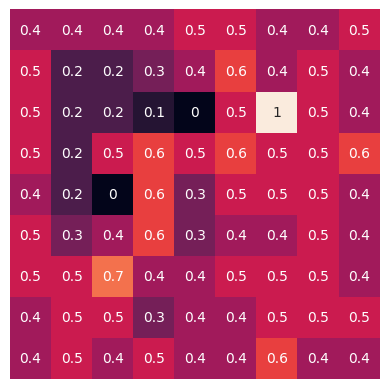} 
           \includegraphics[height=0.14\textwidth]{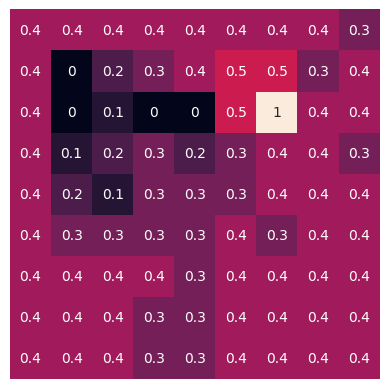} 
            \includegraphics[height=0.14\textwidth]{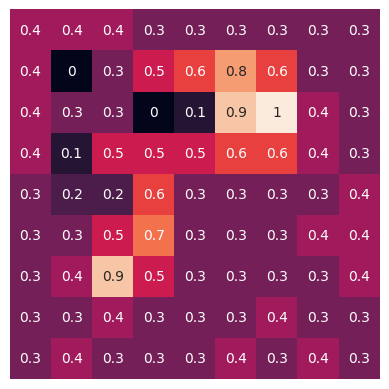}  
            
     \caption{Selected mazes and reward estimates of \drbirl~over the course of 8 environment selections (with two trajectories per environment). We here restrict the set of environments to those that do not block off the expert, i.e., no obstacles can be placed in the top left 4x4 squares. We additionally plot the estimated rewards when running \birl~in each environment separately and taking a coordinate-wise max or coordinate-wise average.} 
     \label{fig:dr_additional_experiment_seed_0}
     
\end{minipage}
\end{figure}

\begin{table}[th!]
\setlength{\tabcolsep}{6pt} 
\renewcommand{\arraystretch}{1.25} 
\centering
\resizebox{\textwidth}{!}{
\begin{tabular}{L{0.14\textwidth}|C{0.14\textwidth}|C{0.14\textwidth}|C{0.14\textwidth}|C{0.14\textwidth}|C{0.14\textwidth}|C{0.14\textwidth}}
\multicolumn{1}{C{0.14\textwidth}|}{Continuous} & \multicolumn{1}{C{0.14\textwidth}}{} &
\multicolumn{1}{C{0.14\textwidth}}{Demo} &
\multicolumn{1}{C{0.14\textwidth}|}{} &
\multicolumn{1}{C{0.14\textwidth}}{} &
\multicolumn{1}{C{0.14\textwidth}}{Test} &
\multicolumn{1}{C{0.14\textwidth}}{} \\ \cline{2-7}
\multicolumn{1}{C{0.14\textwidth}|}{Maze} &
  \multicolumn{1}{C{0.14\textwidth}|}{Average} &
  \multicolumn{1}{C{0.14\textwidth}|}{25\%-quantile} &
  \multicolumn{1}{C{0.14\textwidth}|}{75\%-quantile} &
  \multicolumn{1}{C{0.14\textwidth}|}{Average} &
  \multicolumn{1}{C{0.14\textwidth}|}{25\%-quantile} &
  \multicolumn{1}{C{0.14\textwidth}}{75\%-quantile} \\ \midrule
$\edairl$                              & \textbf{68±04} & \textbf{55±05} & \textbf{97±00} & \textbf{71±02} & \textbf{52±06} & \textbf{99±00}  \\ 
$\drairl$                 & 52±07  &  29±11  & 90±04   & 53±12  & 30±15  & 81±17  \\ \cdashline{1-7}
$\airl$                      & 33±09 & 08±11 & 69±09 & 52±07 & 30±09 & 76±09  \\
$\rime$                                 & -105±12 & -115±05 & -45±02 & -52±03 & -57±03 & -39±02  \\ 
$\gail$                                 & 20±05 & -10±06 & 79±06 & 17±01 & -19±02 & 47±13  \\ 
$\bc$                                   & 11±00 & -23±00 & 71±01 & 22±00 & -18±00 & 49±00   \\
\end{tabular}
}
\newline
\vspace*{15px}
\newline
\resizebox{\textwidth}{!}{
\begin{tabular}{L{0.14\textwidth}|C{0.14\textwidth}|C{0.14\textwidth}|C{0.14\textwidth}|C{0.14\textwidth}|C{0.14\textwidth}|C{0.14\textwidth}}
\multicolumn{1}{C{0.14\textwidth}|}{Hopper} & \multicolumn{1}{C{0.14\textwidth}}{} &
\multicolumn{1}{C{0.14\textwidth}}{Demo} &
\multicolumn{1}{C{0.14\textwidth}|}{} &
\multicolumn{1}{C{0.14\textwidth}}{} &
\multicolumn{1}{C{0.14\textwidth}}{Test} &
\multicolumn{1}{C{0.14\textwidth}}{} \\  \cline{2-7}
\multicolumn{1}{C{0.14\textwidth}|}{} &
  \multicolumn{1}{C{0.14\textwidth}|}{Average} &
  \multicolumn{1}{C{0.14\textwidth}|}{25\%-quantile} &
  \multicolumn{1}{C{0.14\textwidth}|}{75\%-quantile} &
  \multicolumn{1}{C{0.14\textwidth}|}{Average} &
  \multicolumn{1}{C{0.14\textwidth}|}{25\%-quantile} &
  \multicolumn{1}{C{0.14\textwidth}}{75\%-quantile} \\ \midrule
$\edairl$                              & \textbf{63±07} & 37±05 & 76±08 & 52±04 & 31±03 & \textbf{69±08}  \\ 
$\drairl$                 & 59±06 & 35±04 & 69±08 & \textbf{56±04} & 37±04 & 60±06  \\ \cdashline{1-7}
$\airl$                      & 38±03 & 20±02 & 51±05 & 34±04 & 19±02 & 44±07  \\
$\rime$                                 & 61±01 & \textbf{47±01} & \textbf{80±03} & 53±02 & \textbf{42±02} & 66±02   \\ 
$\gail$                                 & 40±02 & 26±01 & 44±03 & 34±01 & 27±01 & 40±01   \\ 
$\bc$                                   & 32±01 & 23±01 & 33±01 & 27±01 & 22±00 & 30±01  \\

\end{tabular}
}
\newline
\vspace*{15px}
\newline

\resizebox{\textwidth}{!}{
\begin{tabular}{L{0.14\textwidth}|C{0.14\textwidth}|C{0.14\textwidth}|C{0.14\textwidth}|C{0.14\textwidth}|C{0.14\textwidth}|C{0.14\textwidth}}
\multicolumn{1}{C{0.14\textwidth}|}{HalfCheetah} & \multicolumn{1}{C{0.14\textwidth}}{} &
\multicolumn{1}{C{0.14\textwidth}}{Demo} &
\multicolumn{1}{C{0.14\textwidth}|}{} &
\multicolumn{1}{C{0.14\textwidth}}{} &
\multicolumn{1}{C{0.14\textwidth}}{Test} &
\multicolumn{1}{C{0.14\textwidth}}{} \\ \cline{2-7}
\multicolumn{1}{C{0.14\textwidth}|}{} &
  \multicolumn{1}{C{0.14\textwidth}|}{Average} &
  \multicolumn{1}{C{0.14\textwidth}|}{25\%-quantile} &
  \multicolumn{1}{C{0.14\textwidth}|}{75\%-quantile} &
  \multicolumn{1}{C{0.14\textwidth}|}{Average} &
  \multicolumn{1}{C{0.14\textwidth}|}{25\%-quantile} &
  \multicolumn{1}{C{0.14\textwidth}}{75\%-quantile} \\ \midrule
$\edairl$                              & \textbf{40±11} & \textbf{17±10} & 61±14 & 35±13 & 20±13 & 46±16  \\ 
$\drairl$                 & \textbf{40±13} & 16±11 & \textbf{63±14} & \textbf{40±11} & \textbf{23±10} & \textbf{54±12}  \\ \cdashline{1-7}
$\airl$                      & 29±09 & 06±12 & 45±09 & 16±07 & 00±07 & 25±08  \\
$\rime$                                 & -21±08 & -29±05 & -12±10 & -10±09 & -11±08 & -06±09   \\  
$\gail$                                 & -12±02 & -28±02 & 01±02 & -06±01 & -14±01 & -06±01   \\  
$\bc$                                   & -23±01 & -30±01 & -16±01 & -14±01 & -17±01 & -12±02    \\
\end{tabular}
}
\newline
\vspace*{15px}
\newline
\resizebox{\textwidth}{!}{
\begin{tabular}{L{0.14\textwidth}|C{0.14\textwidth}|C{0.14\textwidth}|C{0.14\textwidth}|C{0.14\textwidth}|C{0.14\textwidth}|C{0.14\textwidth}}
\multicolumn{1}{C{0.14\textwidth}|}{Swimmer} & \multicolumn{1}{C{0.14\textwidth}}{} &
\multicolumn{1}{C{0.14\textwidth}}{Demo} &
\multicolumn{1}{C{0.14\textwidth}|}{} &
\multicolumn{1}{C{0.14\textwidth}}{} &
\multicolumn{1}{C{0.14\textwidth}}{Test} &
\multicolumn{1}{C{0.14\textwidth}}{} \\ \cline{2-7}
\multicolumn{1}{C{0.14\textwidth}|}{} &
  \multicolumn{1}{C{0.14\textwidth}|}{Average} &
  \multicolumn{1}{C{0.14\textwidth}|}{25\%-quantile} &
  \multicolumn{1}{C{0.14\textwidth}|}{75\%-quantile} &
  \multicolumn{1}{C{0.14\textwidth}|}{Average} &
  \multicolumn{1}{C{0.14\textwidth}|}{25\%-quantile} &
  \multicolumn{1}{C{0.14\textwidth}}{75\%-quantile} \\ \midrule
$\edairl$                             & 80±16 & 51±15 & 110±20 & 69±12 & 53±11 & 90±15  \\ 
$\drairl$               & 45±04 & 22±04 & 62±06 & 53±05 & 36±05 & 71±07  \\ \cdashline{1-7}
$\airl$                      & 40±10 & 11±07 & 66±13 & 44±08 & 19±08 & 68±10  \\ 
$\rime$                                & -05±01 & -08±01 & -03±01 & -04±01 & -07±01 & -02±00 \\ 
$\gail$                                 & 111±00 & 86±01 & 132±00 & 110±01 & 108±00 & 121±00   \\ 
$\bc$                                   & \textbf{124±01} & \textbf{100±01} & \textbf{156±00} & \textbf{130±01} & \textbf{116±00} & \textbf{160±00} \\
\end{tabular}
}
\newline
\vspace*{15px}
\newline

\resizebox{\textwidth}{!}{
\begin{tabular}{L{0.14\textwidth}|C{0.14\textwidth}|C{0.14\textwidth}|C{0.14\textwidth}|C{0.14\textwidth}|C{0.14\textwidth}|C{0.14\textwidth}}
\multicolumn{1}{C{0.14\textwidth}|}{Ant} & \multicolumn{1}{C{0.14\textwidth}}{} &
\multicolumn{1}{C{0.14\textwidth}}{Demo} &
\multicolumn{1}{C{0.14\textwidth}|}{} &
\multicolumn{1}{C{0.14\textwidth}}{} &
\multicolumn{1}{C{0.14\textwidth}}{Test} &
\multicolumn{1}{C{0.14\textwidth}}{} \\ \cline{2-7}
\multicolumn{1}{C{0.14\textwidth}|}{} &
  \multicolumn{1}{C{0.14\textwidth}|}{Average} &
  \multicolumn{1}{C{0.14\textwidth}|}{25\%-quantile} &
  \multicolumn{1}{C{0.14\textwidth}|}{75\%-quantile} &
  \multicolumn{1}{C{0.14\textwidth}|}{Average} &
  \multicolumn{1}{C{0.14\textwidth}|}{25\%-quantile} &
  \multicolumn{1}{C{0.14\textwidth}}{75\%-quantile} \\ \midrule
$\edairl$                              & -71±03 & -79±03 & -46±03 & -86±05 & -98±05 & -55±04  \\ 
$\drairl$                 & -73±03 & -88±03 & -47±03 & -95±03 & -114±05 & -61±03  \\\cdashline{1-7}
$\airl$                      & -73±03 & -89±05 & -47±03 & -97±05 & -113±08 & -62±03  \\ 
$\rime$                                & -43±09 & -64±11 & -23±06 & -45±08 & -59±10 & -27±05 \\ 
$\gail$                                & \textbf{69±03} & 45±03 & \textbf{89±05} & \textbf{75±04} & 44±03 & \textbf{96±04} \\ 
$\bc$                                   & 62±01 & \textbf{49±01} & 72±01 & 69±01 & \textbf{47±02} & 84±02  \\
\end{tabular}
}
\newline
\caption{Normalised scores for the continuous experiments. The quantiles are calculated by taking the average of the quantiles of each individual run.}
\label{tab:detailed_tables_continuous}
\end{table}

\begin{figure}[th!]
     \centering
         \centering
         \includegraphics[width=0.84\textwidth]{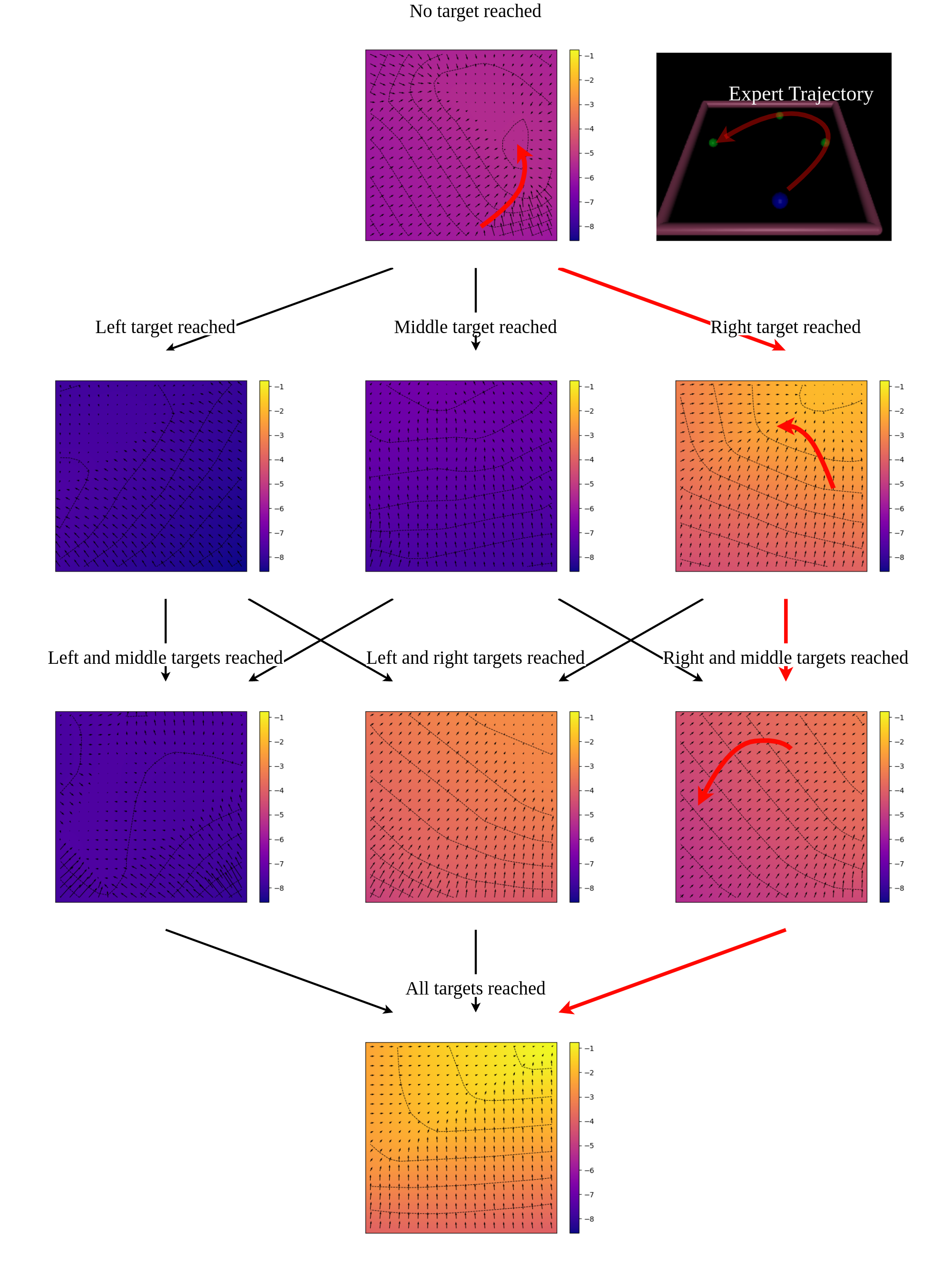} 
        \caption{Reward function learned by $\airl$ on the continuous maze task. The reward function assigns large reward to the trajectories that gather targets in the same order as the expert, i.e., right-middle-left, which corresponds to the right side of the tree. It does not capture that the order in which the green states are reached does not matter. Hence, the learned reward function "overfits" to the expert trajectory. This function may yield undesired behaviors if the agent is forced to gather targets in a different order, or if a target is blocked off. For instance, if the right target is inaccessible, the reward function penalizes the agent for collecting the left target and/or the middle target.\label{fig:heatmaptree_airl}}
        \vspace{23px}
\end{figure}

\begin{figure}[th!]
     \centering
         \centering
         \includegraphics[width=0.84\textwidth]{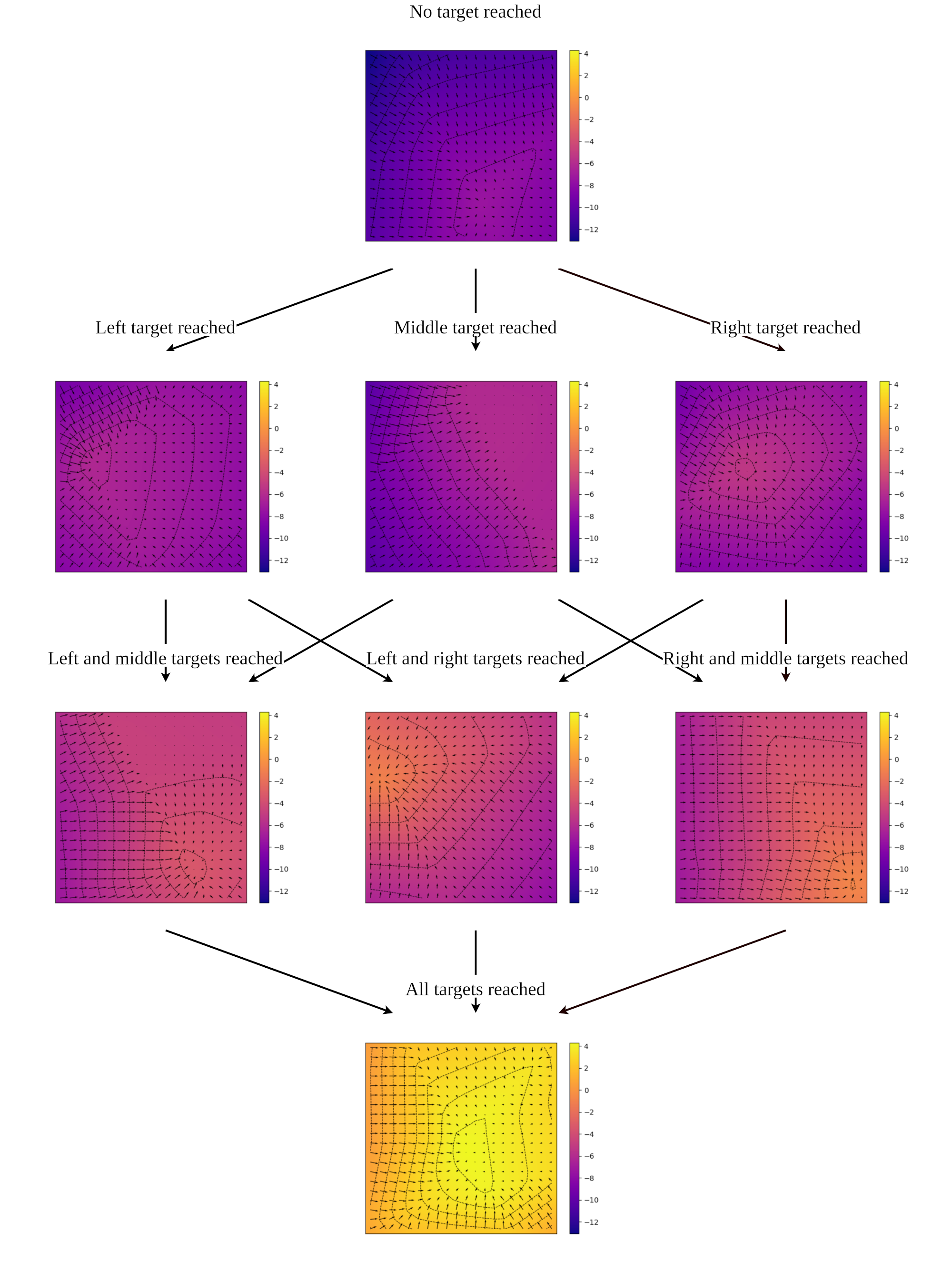} 
        \caption{Reward function learned by $\edairl$ on the continuous maze task. The reward function does not favor a path more than another. It furthermore encourages the agent to collect targets even if some are blocked off.\label{fig:heatmaptree_edairl}}
\vspace{70px}
\end{figure}






\end{document}